\documentclass{article}

\usepackage{microtype}
\usepackage{graphicx}
\usepackage{subfigure}
\usepackage{booktabs} 

\usepackage{amsthm}
\usepackage{amsmath}
\usepackage{bm}
\usepackage{natbib}
\usepackage{mathtools}
\usepackage{amssymb}
\usepackage{color,soul}
\usepackage{todonotes}
\usepackage{xcolor}
\DeclarePairedDelimiterX{\infdivx}[2]{(}{)}{%
  #1\;\delimsize\|\;#2%
}

\newcommand{\norm}[1]{\left\lVert#1\right\rVert}
\usepackage{hyperref}
\newtheorem{theorem}{Theorem}

\newtheorem{lemma}[theorem]{Lemma}
\newtheorem{claim}[theorem]{Claim}
\newtheorem{corollary}[theorem]{Corollary}
\DeclareMathOperator{\Tr}{Tr}
\usepackage[makeroom]{cancel}

\DeclareMathOperator*{\argmin}{argmin}

\usepackage{hyperref}

\usepackage[accepted]{icml2020}

\icmltitlerunning{The Renyi Gaussian Process}

\begin{document}

\twocolumn[
\icmltitle{The R\'enyi Gaussian Process: Towards Improved Generalization} 




\begin{icmlauthorlist}
\icmlauthor{Xubo Yue}{to}
\icmlauthor{Raed Al Kontar}{to}
\end{icmlauthorlist}

\icmlaffiliation{to}{University of Michigan, Ann Arbor}

\icmlcorrespondingauthor{Raed Al Kontar}{alkontar@umich.edu}

\icmlkeywords{Machine Learning, ICML}

\vskip 0.3in
]



\printAffiliationsAndNotice{}  

\begin{abstract}
We introduce an alternative closed form lower bound on the Gaussian process ($\mathcal{GP}$) likelihood based on the R\'enyi $\alpha$-divergence. This new lower bound can be viewed as a convex combination of the Nystr\"om approximation and the exact $\mathcal{GP}$. The key advantage of this bound, is its capability to control and tune the enforced regularization on the model and thus is a generalization of the traditional variational $\mathcal{GP}$ regression. From a theoretical perspective, we provide the convergence rate and risk bound for inference using our proposed approach. Experiments on real data show that the proposed algorithm may be able to deliver improvement over several $\mathcal{GP}$ inference methods. 
\end{abstract}

\section{Introduction}
\label{sec:intro}
The Gaussian process ($\mathcal{GP}$) is a powerful non-parametric learning model that possesses many desirable properties including flexibility, Bayesian interpretation and uncertainty quantification
capability \cite{williams2006gaussian}. $\mathcal{GP}$s have witnessed great successes in various statistics and machine learning areas such as joint predictive modeling \cite{soleimani2017scalable}, Bayesian optimization \cite{snoek2012practical, rana2017high} and deep learning \cite{bui2016deep}.

In $\mathcal{GP}$, model inference (i.e., parameter estimation) is the key part as it will affect prediction/classification accuracy. Currently, inferences of $\mathcal{GP}$s have been mostly based on two approaches \cite{liu2018gaussian}: exact inference and approximate inference mainly via variational inference (VI). Exact inference \citep{williams2006gaussian} directly optimizes the marginal data likelihood function 
\begin{equation*}
    \mathcal{L}_{exact}=\log\mathcal{N}(\bm{0}, \sigma_\epsilon^2 I+\bm{K}_{\bm{f},\bm{f}}),
\end{equation*}
where $\bm{K}_{\bm{f},\bm{f}}$ is full covariance matrix and $\sigma_\epsilon^2$ is the noise parameter. On the other hand, VI \cite{titsias2010bayesian} optimizes a tractable evidence lower bound (ELBO), based on the Kullback-Leibler (KL) divergence, on the data likelihood function. This lower bound is in the form of
\begin{equation*}
    \mathcal{L}_{VI}=\log\mathcal{N}(\bm{0}, \sigma_\epsilon^2 I+\bm{Q})-\frac{1}{2\sigma_\epsilon^2}\text{Tr}(\bm{K}_{\bm{f},\bm{f}}-\bm{Q}),
\end{equation*}
where $\bm{Q}=\bm{K}_{\bm{f},\bm{U}}\bm{K}_{\bm{U},\bm{U}}^{-1}\bm{K}_{\bm{U},\bm{f}}$ is a Nystr\"om low-rank approximation of the exact covariance matrix $\bm{K}_{\bm{f},\bm{f}}$ and $\bm{U}$ is a collection of latent variables. Indeed, this lower bound has seen many success stories as it automatically reduces both computational burden and overfitting due to the enforced regularization on the likelihood. Hitherto, VI has caught most attention \cite{zhao2016variational}, across all approximation inference methods such as expectation propagation and sampling techniques, due to its regularization property and many sound theoretical justifications. Besides prototype VI, there are also some attempts to provide tighter ELBO using normalizing flow \cite{rezende2015variational} or importance-weighted methods \cite{chen2018variational}. See Sec. \ref{sec:related} for a detailed literature review.

However, it is unclear which inference method should be used (i.e., they are data-dependent). Indeed, the recent work of \citet{rainforth2018tighter} discusses an interesting issue about VI and exact inference: is tighter lower bound necessarily better? Through theoretical and empirical evidence, they argue that sometimes a tighter bound is detrimental to the process of learning as it reduces the signal-to-noise ratio (SNR) of estimators. \citet{wang2019exact} also discuss that VI is not necessary better than solving the exact problem. The intuition is as follow: ELBO can be viewed as a smoother to the marginal likelihood. If the likelihood function is very noisy, then many meaningful critical points are obscured by ELBO. On the other hand, attempts to tighten ELBO might suffer from overfitting. Therefore, controlling the tightness of ELBO based on data is necessary and promising.

To provide an approach to resolve this issue, we introduce an alternative closed form lower bound $\alpha$-ELBO on the marginal likelihood function based on the R\'enyi $\alpha$-divergence. This bound has a closed-form 
\begin{align*}
        \mathcal{L}_\alpha=&\log\{\mathcal{N}(\bm{0}, \sigma_\epsilon^2 I+(1-\alpha)\bm{K}_{\bm{f},\bm{f}}+\alpha\bm{Q})\}+\\
        &\log|\bm{I}+\frac{1-\alpha}{\sigma_\epsilon^2}(\bm{K}_{\bm{f},\bm{f}}-\bm{Q})|^{\frac{-\alpha}{2(1-\alpha)}},
\end{align*}
where $|\cdot|$ denotes a determinant operator. $\mathcal{L}_\alpha$ can be viewed as the convex combination of exact $\mathcal{GP}$ ($\bm{K}_{\bm{f},\bm{f}}$) and the variational $\mathcal{GP}$ ($\bm{Q}$), controlled by the tuning parameter $\alpha\in[0,1)$. The tuning parameter determines the shape and tightness of the variational lower bound and is capable of controlling and tuning the enforced regularization on the model. Our proposed bound contains a rich family of $\mathcal{GP}$ inference models. For example, it can be seen that as $\alpha\to 1$, we recover the ELBO $\mathcal{L}_{VI}$ and when $\alpha=0$, we obtain the exact likelihood $\mathcal{L}_{exact}$.


From a theoretical aspect, we first provide the rate of convergence for the R\'enyi $\mathcal{GP}$. Our bound has the form
    \begin{equation*}
        \begin{split}
            &D_{\alpha}[q||p]\\
            &\leq\frac{\alpha }{2\delta(1-\alpha)}\log\bigg[1+\frac{1-\alpha}{\sigma_\epsilon^2}\frac{[(M+1)C+2Nv\epsilon]}{N}\bigg]^N+\\
            &\qquad\qquad \alpha\frac{(M+1)C+2Nv\epsilon}{2\delta\sigma_\epsilon^2}\frac{\norm{\bm{y}}^2}{\sigma_\epsilon^2},
        \end{split}
    \end{equation*}
where the convergence rate can be controlled by parameter $\alpha$. Notably, this bound is connected to the bound on convergence rate derived by Burt et al. \yrcite{burt2019rates} as $\alpha\to1$. This bound can become arbitrarily small as we increase the sample size, number of inducing points or decreasing tuning parameter (refer to Sec. \ref{sec:theory} for detailed notation). We then derive a variational risk bound for the R\'enyi $\mathcal{GP}$
    \begin{align*}
        &\int_{\bm{\Theta}}\big\{r(\bm{\theta},\bm{\theta}^*) p_{\bm{\theta}}(\bm{f}|\bm{U},\bm{\mathcal{Z}}) \big\}d\bm{\theta}\\
        &\leq\frac{\alpha}{n(1-\alpha)}\bigg(\log \frac{P_{\bm{\theta}}(\bm{y})P_{\bm{\theta}^*}(\bm{y})}{\mathcal{N}(\bm{0}, \sigma_\epsilon^2 I+(1-\alpha)\bm{K}_{\bm{f},\bm{f}}+\alpha\bm{Q})} - \\
        &\log |\bm{I}+\frac{1-\alpha}{\sigma_\epsilon^2}(\bm{K}_{\bm{f},\bm{f}}-\bm{Q})|^{\frac{-\alpha}{2(1-\alpha)}}\bigg)+\frac{1}{n(1-\alpha)}\log\frac{1}{\delta}
    \end{align*}
and show that optimizing $\alpha$-ELBO can simultaneously minimize the risk bound and thus is able to yield better parameter estimations (refer to Sec. \ref{sec:theory} for details).

From a computational perspective, we exploit and modify the distributed Blackbox Matrix-Matrix multiplication (BBMM) algorithm proposed by Wang et al. \yrcite{wang2019exact} to efficiently optimize the $\alpha$-ELBO. Experiments on real data show that the proposed algorithm may be able to deliver improvement over several $\mathcal{GP}$ inference methods.

We organize the remaining paper as follows. In Sec. \ref{sec:back}, we briefly review related background knowledge. We then provide the R\'enyi $\mathcal{GP}$ in Sec. \ref{sec:renyi} and its theoretical properties in Sec. \ref{sec:theory}. Detailed literature review about VI can be found in Sec. \ref{sec:related}. In Sec. \ref{sec:exp} we provide numerical experiment to demonstrate the advantages of our model. We conclude our paper in Sec. \ref{sec:conclusion}. Note that most derivations are deferred to the appendix.


\section{Background}
\label{sec:back}
\subsection{Notation}
\label{subsec:notation}
We briefly introduce notation that will be used throughout this paper. Denote by $\bm{\theta}$ and $\bm{\theta}^*$ a vector of parameters of interest and a vector of true parameters, respectively. Assume we have collected $N$ training data points $\bm{y}=(y_i)_{i=1}^N$ with corresponding $D$-dimensional inputs $\bm{x}=(\bm{x}_i)_{i=1}^N$, where $y_i\in\mathbb{R}$ and $\bm{x}_i\in\mathbb{R}^D$. We decompose the output as $y_i=f(\bm{x}_i)+\epsilon_i$, where $f(\cdot)$ is a mean zero $\mathcal{GP}$ and $\epsilon_i(\cdot)$ denotes additive noise with zero mean and $\sigma^2_\epsilon$ variance. Our goal is to predict output $y^*$ given new inputs $\bm{x}^*$. Besides, suppose we have $K$ independent continuous latent variables $\bm{U}=(\bm{U}_{i})_{i=1}^K$ and $M$ inducing inputs $\bm{\mathcal{Z}}=\{z_i\}_{i=1}^M$ \citep{snelson2006sparse}.

\subsection{Review on R\'enyi Divergence}
\label{subsec:renyi}
The R\'enyi's $\alpha$-divergence between two distributions $p$ and $q$ on a random variable (parameter) $\bm{\theta}$ is defined as \cite{renyi1961measures}
\begin{equation*}
    D_\alpha[p||q]=\frac{1}{\alpha-1}\log\int p(\bm{\theta})^\alpha q(\bm{\theta})^{1-\alpha}d\bm{\theta}, \alpha\in[0,1).
\end{equation*}
This divergence contains a rich family of distance measures such as KL-divergence, Bhattacharyya coefficient and $\chi^2$-divergence. Besides, the domain of $\alpha$ can be extended to $\alpha<0$ and $\alpha>1$. By L'H\^opital's rule, it can be easily shown that 
\begin{align*}
    \lim_{\alpha\to 1}D_\alpha[p||q]=KL[p||q]\coloneqq\int p(\bm{\theta})\log\frac{p(\bm{\theta})}{q(\bm{\theta})}d\bm{\theta}.
\end{align*}
Therefore, KL-divergence is a special case of $\alpha$-divergence.  Let $\mathcal{L}_\alpha(q;\bm{y})\coloneqq\log p(\bm{y})-D_\alpha[q(\bm{f},\bm{U}|\bm{\mathcal{Z}})||p(\bm{f},\bm{U}|\bm{y},\bm{\mathcal{Z}})]$, we can then reach the variational R\'enyi (VR) bound. This form is defined as
\begin{equation}
    \label{2:renyi_lb}
    \mathcal{L}_\alpha(q;\bm{y}) \coloneqq \frac{1}{1-\alpha}\log\mathbb{E}_q\bigg[\bigg(\frac{p(\bm{f},\bm{U},\bm{y}|\bm{\mathcal{Z}})}{q(\bm{f},\bm{U}|\bm{\mathcal{Z}})}\bigg)^{1-\alpha}\bigg].
\end{equation}
It can be shown that $\mathcal{L}_0(q;\bm{y})=\log P(\bm{y})$ and $\mathcal{L}_{VI}=\lim_{\alpha\to1}\mathcal{L}_\alpha(q;\bm{y})\leq\mathcal{L}_{\alpha_+}(q;\bm{y})\leq\log P(\bm{y})\leq \mathcal{L}_{\alpha-}(q;\bm{y}), \forall\alpha_+\in(0,1),\alpha_-<0$ \cite{li2016renyi}. See appendix for more properties of the R\'enyi divergence.


\section{The R\'enyi Gaussian Process}
\label{sec:renyi}
Traditional variational inference is seeking to minimize the KL divergence between the variational density $q(\bm{\theta})$ and the intractable posterior $p(\bm{\theta}|\bm{y})$. This minimization problem in turns yields a tractable evidence lower bound of the marginal log-likelihood function of data $\log p(\bm{y})$. The R\'enyi's $\alpha$-divergence is a more general distance measure than the KL divergence. In this section, we want to explore the R\'enyi divergence based $\mathcal{GP}$. Specifically, we will derive a general lower bound, a data-dependent upper bound and provide an efficient algorithm to optimize the lower bound.

\subsection{The Variational R\'enyi Lower Bound}
\label{subsec:v_lbound}
Using the R\'enyi divergence measure, we can obtain a lower bound on the marginal likelihood. Specifically, 
\begin{align}
\label{3:lower_bound}
\begin{split}
        \mathcal{L}_\alpha=&\log\bigg\{\mathcal{N}\big(\bm{0}, \sigma_\epsilon^2 I+(1-\alpha)\bm{K}_{\bm{f},\bm{f}}+\alpha\bm{Q}\big)\bigg\} + \\
        &\log\big|\bm{I}+\frac{1-\alpha}{\sigma_\epsilon^2}(\bm{K}_{\bm{f},\bm{f}}-\bm{Q})\big|^{\frac{-\alpha}{2(1-\alpha)}}.
\end{split}
\end{align}
\begin{proof}
Short guideline: Eq. \eqref{3:lower_bound} is derived by expanding the expectation term in Eq. \eqref{2:renyi_lb} and using the Lyapunov inequality. We can explicitly obtain the optimal distribution of the latent variable $q(\bm{U})$ as 
\begin{equation*}
    q^*(\bm{U})=\frac{p(\bm{y}|\bm{U},\bm{\mathcal{Z}})^{1/(1-\alpha)}p(\bm{U}|\bm{\mathcal{Z}})}{\int p(\bm{y}|\bm{U},\bm{\mathcal{Z}})^{1/(1-\alpha)}p(\bm{U}|\bm{\mathcal{Z}}) d\bm{U}}.
\end{equation*}
Please refer to appendix for a detailed derivation.
\end{proof}

It is clear that the new lower bound is a convex combination of components from sparse $\mathcal{GP}$ ($\bm{Q}$) and components from exact $\mathcal{GP}$ ($\bm{K}_{\bm{f},\bm{f}}$). We can also see that $\alpha$ plays an important role in model regularization \cite{regli2018alpha}. It controls the shape, smoothness and SNR of the lower bound. In fact, the whole $\mathcal{L}_\alpha$ can be viewed as a penalization term. Besides, $\mathcal{L}_\alpha$ is decreasing on $\alpha\in[0,1)$. As $\alpha\to 1$, we recover the well-known ELBO of $\mathcal{GP}$ in the traditional VI.

\subsection{Computation}
\label{subsec:comp}
The exact component $\bm{K}_{\bm{f},\bm{f}}$ in Eq. \eqref{3:lower_bound} is expensive to optimize. To overcome this difficulty, we employ the recently proposed algorithm - Blackbox Matrix-Matrix multiplication \cite{gardner2018gpytorch, gardner2018product, wang2019exact}. The BBMM is an efficient algorithm designed to efficiently solve the exact $\mathcal{GP}$. This algorithm relies on many iterated methods such as conjugate gradient (CG), pivoted cholesky decomposition and parallel computing. Recently, \citet{wang2019exact} have shown that this algorithm can learn $\mathcal{GP}$ with millions of data points using 8 GPU in less than 2 hours.

By scrutinizing Eq. \eqref{3:lower_bound}, we can see that the computation complexity is dominated by the first term, which has the same complexity as the exact $\mathcal{GP}$. The detailed computing procedure is provided as follows. We rewrite Eq. \eqref{3:lower_bound} as
\begin{align}
    \label{3:likelihood}
    \mathcal{L}_\alpha(q;\bm{y})=\log|2\pi\bm{\Xi}|^{-\frac{1}{2}}-\frac{1}{2}\bm{y}^T\bm{\Xi}^{-1}\bm{y}+\log C_x
\end{align}
and gradient can be computed as
\begin{align}
    \label{3:grad}
    \frac{d\mathcal{L}_\alpha(q;\bm{y})}{d\bm{\theta}}&=\frac{1}{2}\bm{y}^T\bm{\Xi}^{-1}\frac{d\bm{\Xi}}{d\bm{\theta}}\bm{\Xi}^{-1}\bm{y}-\frac{1}{2}\Tr\bigg(\bm{\Xi}^{-1}\frac{d\bm{\Xi}}{d\bm{\theta}}\bigg)\nonumber\\
    &-\frac{\alpha}{2(1-\alpha)}\Tr\bigg(A^{-1}\frac{dA}{d\bm{\theta}}\bigg),
\end{align}
where $\Tr$ represents a trace operator and matrices $\bm{\Xi}\coloneqq\sigma_\epsilon^2 I+(1-\alpha)\bm{K}_{\bm{f},\bm{f}}+\alpha\bm{Q}$, $A\coloneqq\bm{I}+\frac{1-\alpha}{\sigma_\epsilon^2}(\bm{K}_{\bm{f},\bm{f}}-\bm{Q})$ and $C_x=\big|\bm{I}+\frac{1-\alpha}{\sigma_\epsilon^2}(\bm{K}_{\bm{f},\bm{f}}-\bm{Q})\big|^{\frac{-\alpha}{2(1-\alpha)}}$.

In Eq. \eqref{3:likelihood} and \eqref{3:grad}, three expensive terms $\log|\bm{\Xi}|$, $\bm{\Xi}^{-1}\bm{y}$ and $\Tr\big(\bm{\Xi}^{-1}\frac{d\bm{\Xi}}{d\bm{\theta}}\big)$ can be efficiently estimated by Batched Conjugate Gradients Algorithm (mBCG) \cite{gardner2018gpytorch} with some modifications. The remaining work is to estimate the second term in Eq. \eqref{3:lower_bound}. First, we can write it as
\begin{align*}
    \log C_x&=\log|\bm{I}+\frac{1-\alpha}{\sigma_\epsilon^2}(\bm{K}_{\bm{f},\bm{f}}-\bm{Q})|^{\frac{-\alpha}{2(1-\alpha)}}\\
    &=\log|\frac{\bm{\Xi}}{\sigma_\epsilon^2}+\frac{1-2\alpha}{\sigma_\epsilon^2}\bm{Q}|^{\frac{-\alpha}{2(1-\alpha)}}.
\end{align*}
By Matrix determinant lemma, we have
\begin{align*}
     \log&|\frac{\bm{\Xi}}{\sigma_\epsilon^2}+\frac{1-2\alpha}{\sigma_\epsilon^2}\bm{Q}|=\log|\frac{1}{\sigma_\epsilon^2}||\bm{\Xi}+(1-2\alpha)\bm{Q}|\\
     &=\log|\frac{1}{\sigma_\epsilon^2}||\bm{\Xi}+(1-2\alpha)\bm{K}_{\bm{f},\bm{U}}\bm{K}_{\bm{U},\bm{U}}^{-1}\bm{K}_{\bm{U},\bm{f}}|\\
     &=\log|\frac{1}{\sigma_\epsilon^2}||\frac{1}{(1-2\alpha)}\bm{K}_{\bm{U},\bm{U}}+\bm{K}_{\bm{U},\bm{f}}\bm{\Xi}^{-1}\bm{K}_{\bm{f},\bm{U}}|+\\
     &\qquad\log|(1-2\alpha)\bm{K}_{\bm{U},\bm{U}}^{-1}|+\log|\bm{\Xi}|.  
\end{align*}
In this equation, $\log|\bm{\Xi}|$ is already available as aforementioned. Therefore, only $\bm{\Xi}^{-1}\bm{K}_{\bm{f},\bm{U}}$ is expensive to compute. Similarly, we resort to CG algorithm to overcome this difficulty. Overall, the resulting matrix is of dimension $M\times M$ (note that $M\ll N$) and is cheap to compute. On the other hand, in the gradient part, we have
\begin{align*}
    &\Tr(A^{-1}\frac{dA}{d\bm{\theta}}) = \Tr\bigg(\bigg(\frac{\bm{\Xi}}{\sigma_\epsilon^2}+\frac{1-2\alpha}{\sigma_\epsilon^2}\bm{Q}\bigg)^{-1}\frac{dA}{d\bm{\theta}}\bigg)\\
    &=\Tr\bigg(\big(\frac{\bm{\Xi}}{\sigma_\epsilon^2}\big)^{-1}\frac{dA}{d\bm{\theta}} - \big(\frac{\bm{\Xi}}{\sigma_\epsilon^2}\big)^{-1}\bm{K}_{\bm{f},\bm{U}}\big(\frac{\sigma^2_\epsilon}{1-2\alpha}\bm{K}_{\bm{U},\bm{U}}\\
    & \qquad + \bm{K}_{\bm{U},\bm{f}}\frac{\bm{\Xi}^{-1}}{\sigma^2_\epsilon}\bm{K}_{\bm{f},\bm{U}}\big)^{-1}  \bm{K}_{\bm{U},\bm{f}}\big(\frac{\bm{\Xi}}{\sigma_\epsilon^2}\big)^{-1}\frac{dA}{d\bm{\theta}}\bigg)
\end{align*}
by Woodbury matrix identity. In above, all components are available without further heavy computation requirement.

The detailed derivation and implementation are deferred to the appendix. Note that besides BBMM, there are also many other promising approaches to optimize $\alpha$-ELBO such as stochastic VI \cite{hoffman2013stochastic}, distributed VI \cite{gal2014distributed}, Stein method \cite{liu2016stein}, black box variational method \cite{tran2015variational} or doubly stochastic VI \cite{salimbeni2017doubly}. However, our goal is not to conduct an exhaustive comparison study about merits of those methods. We will explore, in the future, that which inference method can optimally speed up parameter estimation.

\subsection{Tuning Parameter $\alpha$} 
\label{subsec:tuning}
We use cross validation to choose $\alpha$ values \cite{li2016renyi, bui2017unifying}. As we will show in Sec. \ref{sec:exp}, a moderate $\alpha$ (i.e., close to 0.5) works well in both simulated and real data. This is intuitively understandable as 0.5 balances between the two spectrum's of exact and KL based inference. For example, \citet{kailath1967divergence} has shown that, for a 0-1 classification problem, Bhattacharyya coefficient ($\alpha=0.5$) gives tight lower and upper bounds on the error probability.

\subsection{Prediction}
After estimating parameter $\bm{\theta}$, we can predict $\bm{y}$ given new input data points. In R\'enyi $\mathcal{GP}$, the predictive distribution for a new input $\bm{x}^*$ is given by
\begin{equation}
\label{3:23}
    \begin{split}
        &p(y^*|\bm{y})=\int p(y^*|\bm{U})p(\bm{U}|\bm{y})d\bm{U}\\
        &=\int \mathcal{N}(\bm{K}_{\bm{f^*},\bm{U}}\bm{K}_{\bm{U},\bm{U}}^{-1}\bm{U},\bm{\Xi})p(\bm{U}|\bm{y})d\bm{U}\\
        &=\int \mathcal{N}(\bm{K}_{\bm{f^*},\bm{U}}\bm{K}_{\bm{U},\bm{U}}^{-1}\bm{U},\bm{\Xi})\frac{p(\bm{y}|\bm{U})p(\bm{U})}{p(\bm{y})}d\bm{U}\\
        &=\mathcal{N}(\bm{A\Xi}^{-1}\bm{y}, \bm{K}_{\bm{f}^*,\bm{f}^*}+\sigma_\epsilon^2\bm{I}-\bm{A\Xi}^{-1}\bm{A}^T),
    \end{split}
\end{equation}
where 
\begin{align*}
    \bm{A}&=\bm{K}_{\bm{f^*},\bm{U}}\bm{K}_{\bm{U},\bm{U}}^{-1}\bm{K}_{\bm{U},\bm{f^*}}\\
    p(\bm{y}|\bm{U})&=\mathcal{N}(\bm{K}_{\bm{f},\bm{U}}\bm{K}_{\bm{U},\bm{U}}^{-1}\bm{U},\sigma_\epsilon^2\bm{I})
\end{align*}
and we have used $\bm{K}_{\bm{f^*},\bm{f^*}}$ as a notation to indicate when the covariance matrix is evaluated at the $x^*$. Consequently, the predicted trajectories have mean $\bm{A\Xi}^{-1}\bm{y}$ and variance $\bm{K}_{\bm{f}^*,\bm{f}^*}+\sigma_\epsilon^2I-\bm{A\Xi}^{-1}\bm{A}^T$. 


\section{Theoretical Properties}
\label{sec:theory}
In this section, we study the rate of convergence and risk bound on the R\'enyi Gaussian process.
\subsection{A Data-dependent Upper Bound} 
In order to derive the convergence rate, we need to obtain a data-dependent upper bound on the marginal likelihood. Titsias   \yrcite{titsias2014variational} provides a bound based on the KL divergence. We can generalize this bound into $\mathcal{L}_{upper}=$
\begin{align}
\label{3:upper_bound}
        \log\frac{1}{|2\pi\bm{\Xi}|^{\frac{1}{2}}} -\frac{1}{2}\bm{y}^T\big(\bm{\Xi}+\alpha\text{Tr}(\bm{K}_{\bm{f},\bm{f}}-\bm{Q})\bm{I}\big)^{-1}\bm{y},
\end{align}
where $\bm{\Xi}\coloneqq\sigma_\epsilon^2 I+(1-\alpha)\bm{K}_{\bm{f},\bm{f}}+\alpha\bm{Q}$. 
\begin{proof}
Short guideline: the proof is first based on a property of positive semi-definite (PSD) matrix. Suppose $A$ and $B$ are PSD matrices and $A-B$ is also PSD. Then $|A|\geq|B|$. Furthermore, if $A$ and $B$ are positive definite (PD), then $B^{-1}-A^{-1}$ is also PD \cite{horn2012matrix}. The rest of the proof is followed by eigen-decomposition and some algebraic manipulations. Please refer to appendix for a detailed derivation.
\end{proof}

\subsection{Rate of Convergence}
\label{subsec:conv}
In this section, we provide the rate of convergence of the R\'enyi $\mathcal{GP}$.
\begin{theorem}
\label{main1}
Suppose $N$ data points are drawn independently from input distribution $p(\bm{x})$ and $k(\bm{x},\bm{x})\leq v_0, \forall \bm{x}\in\mathcal{X}$. Sample $M$ inducing points from the training data with the probability assigned to any set of size $M$ equal to the probability assigned to the corresponding subset by an $\epsilon$ k-Determinantal Point Process (k-DPP) \citep{belabbas2009spectral} with $k=M$. If $\bm{y}$ is distributed according to a sample from the prior generative model, then with probability at least $1-\delta$,
\begin{equation*}
    \begin{split}
        D_\alpha&[p||q]\leq\alpha\frac{(M+1)C+2Nv_0\epsilon}{2\delta\sigma_\epsilon^2}+\\
        &\frac{1}{\delta}\frac{\alpha }{2(1-\alpha)}\log\bigg[1+\frac{1-\alpha}{\sigma_\epsilon^2}\frac{[(M+1)C+2Nv_0\epsilon]}{N}\bigg]^N.
    \end{split}
\end{equation*}
where $C=N\sum_{m=M+1}^\infty\lambda_m$ and $\lambda_m$ are the eigenvalues of the integral operator $\mathcal{K}$ associated to kernel and $p(\bm{x})$.
\end{theorem}

\begin{theorem}
    \label{main2}
    Suppose $N$ data points are drawn independently from input distribution $p(\bm{x})$ and $k(\bm{x},\bm{x})\leq v_0, \forall \bm{x}\in\mathcal{X}$. Sample $M$ inducing points from the training data with the probability assigned to any set of size $M$ equal to the probability assigned to the corresponding subset by an $\epsilon$ k-DPP with $k=M$. With probability at least $1-\delta$,
    \begin{equation*}
        \begin{split}
            &D_{\alpha}[q||p]\\
            &\leq\frac{\alpha }{2\delta(1-\alpha)}\log\bigg[1+\frac{1-\alpha}{\sigma_\epsilon^2}\frac{[(M+1)C+2Nv_0\epsilon]}{N}\bigg]^N+\\
            &\qquad \alpha\frac{(M+1)C+2Nv_0\epsilon}{2\delta\sigma_\epsilon^2}\frac{\norm{\bm{y}}^2}{\sigma_\epsilon^2}.
        \end{split}
    \end{equation*}
\end{theorem}
\begin{proof}
Short guideline: we first bound the regularization term by
\begin{align*}
    -\log C_x\leq \frac{\alpha}{2(1-\alpha)}\log \bigg(\frac{\text{Tr}(\bm{I}+\frac{1-\alpha}{\sigma_\epsilon^2}(\bm{K}_{\bm{f},\bm{f}}-\bm{Q}))}{N}\bigg)^N.
\end{align*}
Then we bound the R\'enyi divergence by
\begin{equation*}
    \begin{split}
        &\log |\bm{I}+\frac{1-\alpha}{\sigma_\epsilon^2}(\bm{K}_{\bm{f},\bm{f}}-\bm{Q})|^{\frac{\alpha}{2(1-\alpha)}} \leq \mathbb{E}_y\bigg[D[q||p]\bigg]\leq \\
        &\log |\bm{I}+\frac{1-\alpha}{\sigma_\epsilon^2}(\bm{K}_{\bm{f},\bm{f}}-\bm{Q})|^{\frac{\alpha}{2(1-\alpha)}} + \frac{\alpha\text{Tr}(\bm{K}_{\bm{f},\bm{f}}-\bm{Q})}{2\sigma_\epsilon^2}.
    \end{split}
\end{equation*}
The remaining step is to find a bound on the expectation of $D_\alpha[p||q]$ with respect to data $\bm{y}$, inducing points $\bm{\mathcal{Z}}$ and input distribution $\bm{X}$. This part is tedious and we move the detailed technical proofs into the appendix.
\end{proof}
Theorem \ref{main1} and \ref{main2} imply that $D_\alpha[p||q]$ can be made arbitrarily small with high probability. The rate of this convergence can be controlled by sample size, number of inducing variables, decay rate of eigenvalues and tuning parameter $\alpha$. See Sec. \ref{subsec:consequence} for more examples.


\subsection{Consequences}
\label{subsec:consequence}
Based on Theorem \ref{main1} and \ref{main2}, we can derive the convergence rate for smooth (e.g., the square exponential kernel) and non-smooth (e.g., Mat\'ern) kernels. 

\subsubsection{Smooth Kernel}
We will provide a convergence result with the square exponential (SE) kernel. The $m$-th eigenvalue of kernel operator is $\lambda_m=v\sqrt{2a/A}B^{m-1}$, where $a=1/(4\sigma_\epsilon^2)$, $b=1/(2\ell^2)$, $c=\sqrt{a^2+2ab}$, $A=a+b+c$ and $B=b/A$. $\ell$ is the length parameter, $v$ is signal variance and $\sigma_\epsilon$ is the noise parameter. We can obtain $\sum_{m=M+1}^\infty\lambda_m=\frac{v\sqrt{2a}}{(1-B)\sqrt{A}}B^M$.

\begin{corollary}
    Suppose $\norm{\bm{y}}^2\leq RN$, where $R$ is a constant. Fix $\gamma>0$ and take $\epsilon=\frac{\delta\sigma_\epsilon^2}{vN^{\gamma+2}}$. Assume the input data is normally distributed and regression in performed with a SE kernel. With probability $1-\delta$,
    \begin{align*}
        &D_\alpha[p||q]\\
        &\leq 2\alpha\frac{R}{\sigma_\epsilon^2}\frac{1}{N^\gamma}+\frac{1}{\delta}\frac{\alpha}{2(1-\alpha)}\log\bigg[1+(1-\alpha)\big(\frac{4\delta}{N^{\gamma+2}} \big)\bigg]^N,  
    \end{align*}
    when inference is performed with $M=\frac{(3+\gamma)\log N+\log \eta}{\log(B^{-1})}$, where $\eta=\frac{v\sqrt{2a}}{a\sqrt{A}\sigma_\epsilon^2\delta(1-B)}$.
\end{corollary}
\begin{proof}
    We know $\frac{C(M+1)}{2\delta\sigma_\epsilon^2}<\frac{1}{N^{\gamma+1}}$. By Theorem \ref{main2}, we can obtain the following bound
    \begin{equation*}
        \begin{split}
            &D_\alpha[p||q]\\
            &\leq 2\alpha\frac{R}{\sigma_\epsilon^2}\frac{1}{N^\gamma}+\frac{1}{\delta}\frac{\alpha}{2(1-\alpha)}\log\bigg[1+(1-\alpha)\big(\frac{4\delta}{N^{\gamma+2}}\big)\bigg]^N\\
            &< 2\alpha\frac{R}{\sigma_\epsilon^2}\frac{1}{N^\gamma}+\alpha \big(\frac{2}{N^{\gamma+1}} \big)=\frac{\alpha}{N^\gamma}(\frac{2R}{\sigma_\epsilon^2}+\frac{2}{N}).
        \end{split}
    \end{equation*}
\end{proof}
This corollary has two implications. First, it implies that the number of inducing points should be of order $\mathcal{O}(\log N)$ (i.e., sparse). In high dimension input space, following a similar proof, we can show that this order becomes $\mathcal{O}(\log^D N)$. Second, the tuning parameter $\alpha$ plays an important role in controlling convergence rate. A small $\alpha$ ensures fast convergence to true posterior but might also decrease SNR. In Sec. \ref{sec:exp}, we will show that a moderate $\alpha$ value is promising.

\subsubsection{Non-smooth Kernel}
For the Mat\'ern $r+\frac{1}{2}$, $\lambda_m\asymp\frac{1}{m^{2r+2}}$, where $\asymp$ means ``asymptotically equivalent to''. We can obtain $\sum_{m=M+1}^\infty\lambda_m=\mathcal{O}(\frac{1}{M^{2r+1}})$ by the following claim.
\begin{claim}
    $\sum_{m=M+1}^\infty\lambda_m=\mathcal{O}(\frac{1}{M^{2r+1}})$.
\end{claim}
\begin{proof}
It is easy to see that $\sum_{m=1}^\infty\lambda_m=\zeta(2r+2)$, where $\zeta$ is a Riemann zeta function. By the Euler-Maclaurin sum formula, we have the generalized harmonic number \citep{woon1998generalization}
\begin{align*}
    &\sum_{m=1}^M(\frac{1}{m})^{2r+2}=\zeta(2r+2)+\frac{1}{-2r-1}M^{-2r-1}+\\
    &\qquad\frac{1}{2}M^{-2r-2}-\frac{2r+2}{12}M^{-2r-3}+\mathcal{O}(M^{-2r-4}).
\end{align*}
Therefore, 
\begin{align*}
    &\sum_{m=M+1}^\infty\lambda_m=-\frac{1}{-2r-1}M^{-2r-1}-\frac{1}{2}M^{-2r-2}+\\
    &\qquad\frac{2r+2}{12}M^{-2r-3}-\mathcal{O}(M^{-2r-4})=\mathcal{O}(\frac{1}{M^{2r+1}}).
\end{align*}
\end{proof}
Let $\sum_{m=M+1}^\infty\lambda_m\leq A\frac{1}{M^{2r+1}}$. Then by Theorem \ref{main2}, we have
\begin{equation*}
    \begin{split}
        &\alpha\frac{(M+1)N\sum_{m=M+1}^\infty\lambda_m+2Nv_0\epsilon}{2\delta\sigma_\epsilon^2}\frac{\norm{\bm{y}}^2}{\sigma_\epsilon^2}\\
        &\leq\alpha\frac{(M+1)NA\frac{1}{M^{2k+1}}+2Nv_0\epsilon}{2\delta\sigma_\epsilon^2}\frac{RN}{\sigma_\epsilon^2}\\
        &=\frac{\alpha R}{2\delta\sigma_\epsilon^4}\big(\frac{(M+1)N^2A}{M^{2r+1}}+2N^2v_0\epsilon\big).
    \end{split}
\end{equation*}
In order to let $\lim_{N\to\infty}\frac{(M+1)N^2}{M^{2r+1}}\to 0$, we require $M=N^t$ ($t$ will be clarified shortly). Therefore, 
\begin{equation*}
    \begin{split}
        \frac{(M+1)N^2A}{M^{2r+1}}=\frac{(N^t+1)N^2A}{N^{(2r+1)t}}\leq\frac{A}{N^{2rt-2}}. 
    \end{split}
\end{equation*}
Let $2rt-2\geq \gamma$, then $t\geq\frac{\gamma+2}{2r}$. Therefore, we have
\begin{equation*}
    \frac{\alpha R}{2\sigma_\epsilon^4}\big(\frac{(M+1)N^2A}{M^{2r+1}}+2N^2v_0\epsilon\big)\leq\frac{\alpha R}{N^\gamma\sigma_\epsilon^2}+\frac{\alpha RA}{2\delta\sigma_\epsilon^4N^\gamma}.
\end{equation*}
Another term in the bound can also be simplified as
\begin{equation*}
    \begin{split}
        &\frac{\alpha }{2(1-\alpha)}\log\bigg[1+\frac{1-\alpha}{\sigma_\epsilon^2}\frac{[(M+1)C+2Nv_0\epsilon]}{N}\bigg]^N\\
        &\leq\frac{\alpha N}{2(1-\alpha)}\log\bigg[1+(1-\alpha)\big(\frac{A}{\sigma_\epsilon^2N^{\gamma+2}}+\frac{2\delta}{\sigma_\epsilon^2N^{\gamma+2}} \big) \bigg].
    \end{split}
\end{equation*}
It can be seen that we require more inducing points ($\mathcal{O}(N^t)$) when we are using non-smooth kernels and $t$ is decreasing as we increase the smoothness (i.e., $r$) of Mat\'ern kernel. Besides, we can also see that $\alpha$ is crucial in the convergence rate.

\subsection{Bayes Risk Bound}
\label{subsec:risk}
The Bayes risk is defined as $\mathcal{R}=\mathbb{E}[r(\bm{\theta},\bm{\theta}^*)]=\int r(\bm{\theta},\bm{\theta}^*) p_{\bm{\theta}}(\bm{f}|\bm{U},\bm{\mathcal{Z}}) d\bm{\theta}$. Bayes Risk is of interest in a broad scope of machine learning problems. For example, in the sparse linear regression, we estimate parameters by minimizing a squared loss \cite{chen2016bayes} or absolute residual over Wasserstein ball \cite{chen2018robust}.
\begin{theorem}
    \label{thm:risk}
    With probability at least $1-\delta$,
    \begin{align*}
        &\int_{\bm{\Theta}}\big\{r(\bm{\theta},\bm{\theta}^*) p_{\bm{\theta}}(\bm{f}|\bm{U},\bm{\mathcal{Z}}) \big\}d\bm{\theta}\\
        &\leq\frac{\alpha}{n(1-\alpha)}\bigg(\log \frac{P_{\bm{\theta}}(\bm{y})P_{\bm{\theta}^*}(\bm{y})}{\mathcal{N}(\bm{0}, \sigma_\epsilon^2 I+(1-\alpha)\bm{K}_{\bm{f},\bm{f}}+\alpha\bm{Q})} - \\
        &\log |\bm{I}+\frac{1-\alpha}{\sigma_\epsilon^2}(\bm{K}_{\bm{f},\bm{f}}-\bm{Q})|^{\frac{-\alpha}{2(1-\alpha)}}\bigg)+\frac{1}{n(1-\alpha)}\log\frac{1}{\delta}.
    \end{align*}
\end{theorem}
\begin{proof}
Short guideline: this upper bound is followed by applying many concentration inequalities. The probability component is obtained from Markov’s inequality. Please refer to appendix for details.
\end{proof}
Based on this expression, we can see that maximizing $\mathcal{L}_\alpha(q;\bm{y})$ is equivalent to minimizing the R\'enyi divergence and Bayes risk. 
Interestingly, this risk bound cannot be extended to the KL divergence case if we simply take $\alpha\to1$ limit operation \cite{ghosal2007convergence}. The valid risk bound for KL divergence requires strong assumptions on the identifiability and prior concentration \cite{ghosal2000convergence, yang2017alpha, alquier2017concentration}.

Deriving a generalization bound or error bound for the R\'enyi $\mathcal{GP}$ based on Theorem \ref{thm:risk} is an interesting topic \cite{chen2016bayes, wang2019prediction}. We will pursue this direction in the future.


\section{Related Work}
\label{sec:related}
In the machine learning community, recent advances of learning \textbf{Gaussian process} follow three major trends. \textbf{First}, sampling methods such as Markov chain Monte Carlo (MCMC) \cite{frigola2013bayesian, hensman2015mcmc} and Hamiltonian Monte Carlo \cite{havasi2018inference} have been extensively studied recently. Sampling approaches are developed to capture the posterior distribution of non-Gaussian or multi-modal functions. However, a  sampling approximation is usually computationally intensive. Notably, a recent comparison study \cite{lalchand2019approximate} shows that VI can achieve a remarkably comparable performance to sampling approach while the former one has better theoretical properties and can be fitted into many existing efficient optimization frameworks. 

\textbf{Second}, the expectation propagation (EP) \cite{deisenroth2012expectation} is an iterative local message passing method designed for approximate Bayesian inference. Based on this approach, \citet{bui2017unifying} propose a generalized EP (power EP) framework to learn $\mathcal{GP}$ and demonstrate that power EP encapsulates a rich family of approximated $\mathcal{GP}$ such as FITC and DTC \cite{bui2017unifying}. Though accurate and promising, the EP family, in general, is not guaranteed to converge \cite{bishop2006pattern}. Therefore, EP has caught relatively few attention in the machine learning community.

\textbf{Third}, variational inference is an approach to estimate probability densities through efficient optimization algorithms \cite{hoffman2013stochastic, hoang2015unifying, blei2017variational}. It approximates intractable posterior distribution using a tractable distribution family $\mathcal{Q}$. This approximation in turn yields a closed-form ELBO and this bound can be used to learn parameters from model. Hitherto, VI has caught more attention than many other approximation inference algorithms due to its elegant and useful theoretical properties. 

Besides those three major trends, there are also other approaches using tensor decomposition \cite{nickson2015blitzkriging, izmailov2017scalable} or local approximation. For a comprehensive review of approximation methods in $\mathcal{GP}$, please refer to recent work from \citet{liu2018gaussian}.



\section{Experiment}
\label{sec:exp}


\begin{figure*}[!htb]
    \begin{center}
    \centerline{\includegraphics[width=1.6\columnwidth]{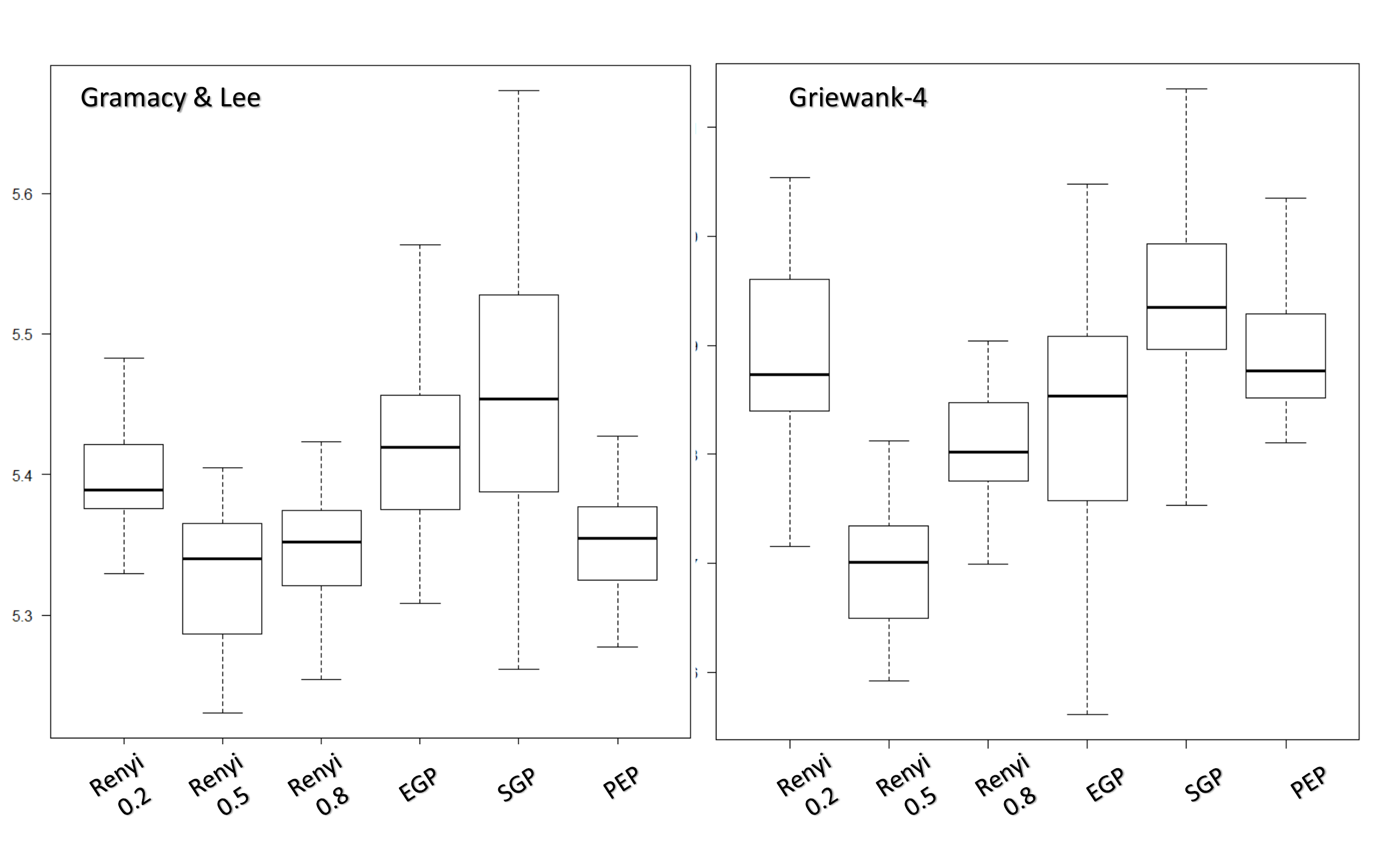}}
    \caption{Boxplots of RMSE on simulation datasets.}
    \label{result:simulation}
    \end{center}
\end{figure*}

\begin{table*}[!htb]
  \caption{RMSE of all models on many datasets. The RMSE is calculated over 20 experiments with different initial points. For the R\'enyi $\mathcal{GP}$, we also report the optimal $\alpha$ value.}
  \vspace{0.5cm}
  \label{result:rmse}
  \centering
  \begin{tabular}{cccccc}
    \toprule
    Dataset & EGP & SGP & PEP & R\'enyi & Optimal $\alpha$ \\
    \midrule
    Bike & $13.41\pm1.23$ & $16.93\pm3.33$ & $14.70\pm1.61$ & $11.16\pm1.74$ & $0.50$ \\
    C-MAPSS & $16.11\pm1.15$ & $17.45\pm1.66$ & $15.03\pm1.00$ & $13.03\pm0.47$ & $0.45$  \\
    PM2.5 &  $11.74\pm0.81$ & $15.85\pm1.03$ & $10.83\pm0.94$ & $8.02\pm0.55$ & $0.55$   \\
    Traffic & $15.42\pm1.42$ & $17.47\pm1.42$ & $15.17\pm1.05$ & $12.85\pm1.40$ & $0.50$    \\
    Battery & $20.16\pm1.06$ & $29.96\pm1.09$ & $21.11\pm1.04$ & $9.90\pm1.10$ & $0.50$   \\
    \bottomrule
  \end{tabular}
\end{table*}

\subsection{Benchmark Models}

We benchmark our model with the exact and scalable $\mathcal{GP}$ inference (EGP) \cite{wang2019exact}, the sparse variational $\mathcal{GP}$ (SGP) \cite{titsias2010bayesian} and the power expectation propagation (PEP) \cite{bui2017unifying} with $\alpha=0.5$. PEP unifies a large number of pseudo-point approximations such as FITC and DTC.

\subsection{A Toy Example}
We first investigate the performance of R\'enyi $\mathcal{GP}$ method on some simulated toy regression datasets with 1,000 data points in various dimensions. The data are from Virtual Library of Simulation Experiments (\url{http://www.sfu.ca/~ssurjano/index.html}). The testing functions are Gramacy \& Lee function ($D=1$),  Branin-Hoo function ($D=2$) and Griewank-$D$ function ($D\geq 2$). For each dataset, we randomly split 60\% data as training sets and 40\% as testing sets. We set number of inducing points to be 50. Throughout the experiment, we use mean 0 and SE kernel $\mathcal{GP}$ prior. For each function, we run our model 30 times with different $\alpha\in[0.2, 0.8]$ and initial parameters. The performance of each model is measured by Root Mean Square Error (RMSE).

\subsection{Results on Simulation Data}
Due to space limit we only report results from $D=1$ and $D=4$ in Plot \ref{result:simulation}. The results clearly indicate that our model, in general, has the smallest RMSE among all benchmark models. When $\alpha\approx0.5$, we achieve the smallest RMSE. When $\alpha$ is around 0.2, the RMSE is compromised. \textbf{This evidences the danger of ambitiously tightening ELBO}.

\subsection{Real Data}
We compare the performance of the R\'enyi $\mathcal{GP}$ against other inference methods on a range of datasets from (1) the UCI data repository \cite{asuncion2007uci}  (\url{https://archive.ics.uci.edu/ml/datasets.php}), (2) the battery data from the General Motors Onstar System and (3) the C-MAPSS aircraft turbofan engines
dataset provided by the National Aeronautics and Space Administration (NASA) (\url{https://ti.arc.nasa.gov/tech/dash/groups/pcoe/}). We only focus on regression tasks. Our goal is to demonstrate that the additional parameter $\alpha$ improves the flexibility and thus the prediction performance of Gaussian process.

\begin{figure}[!htb]
    \vskip 0.1in
    \begin{center}
    \centerline{\includegraphics[width=0.8\columnwidth]{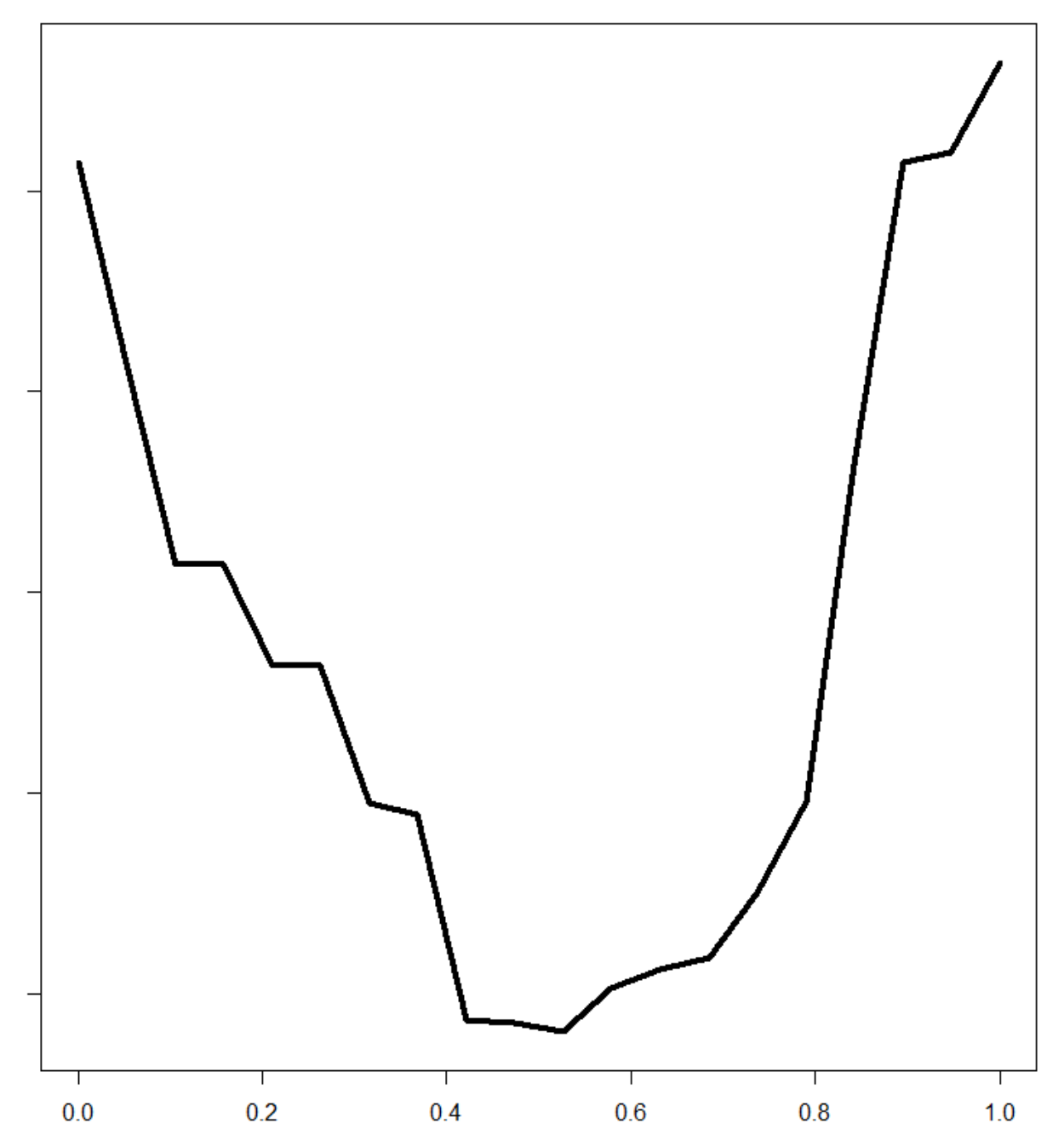}}
    \caption{RMSE vs. $\alpha$}
    \label{result:alpha}
    \end{center}
    \vskip -0.4in
\end{figure}

Our data contain the bike sharing dataset (Bike, $N=17,389$), the aircraft turbofan engines degradation signal data (C-MAPSS, $N=33,727$), Beijing PM2.5 data (PM2.5, $N=43,824$), Metro interstate traffic volume dataset (Traffic, $N=48,204$) and battery data (Battery, $N=104,046$). Note that $N$ contains both training and testing data. Overall, the size of data ranges from 10,000 to 100,000.

For each dataset, we randomly split 60\% data as training sets and 40\% as testing sets. We set number of inducing points to be sparse (i.e., $M=\mathcal{O}(\log^D n)$) based on our convergence result. All data are standardized to be mean 0 and variance 1.

We study the effect of $\alpha$ on the prediction performance. For each dataset, we run R\'enyi $\mathcal{GP}$ with different $\alpha\in\{0.3, 0.35, \ldots, 0.65, 0.7\}$ and select optimal $\alpha$ with the smallest RMSE. Here we note that a theoretical guideline on choosing optimal $\alpha$ values is needed and we leave it as a future work.

We use Method of Moving Asymptotes (MMA) with gradient information to optimize all hyperparameters $\bm{\theta}$ (excluding $\alpha\in[0,1)$). The upper bound for number of iterations is set to be $10,000$. For mBCG algorithm, we use Diagonal Scaling preconditioning matrix to stabilize algorithm and boost convergence speed \cite{takapoui2016preconditioning}. In mBCG, the maximum number of iterations is set to be $10N$. 

The matrix multiplication process is distributed in parallel into 2 GPU. However, the parallel computing is not mandatory and one can resort to single GPU due to limited computing resource. We code R\'enyi $\mathcal{GP}$ in Rstudio (version 3.5.0). An illustrative code is provided in the supplementary material.

\subsection{Results and Discussion}

Experimental results are reported in Table \ref{result:rmse}. The performance of each model is measured by RMSE. The RMSE is calculated over 20 experiments with different initial points. We also report standard deviation (std). Based on Figure \ref{result:simulation} and Table \ref{result:rmse}, we can obtain some important insights. 

First, the results indicate that our models achieve the smallest RMSE among all benchmarks on all datasets ranging from small data ($N\approx1,000$) to moderately big data ($N\approx100,000$). The key reason is that tuning parameter $\alpha$ introduces an additional flexibility on the model inference. 

Second, by empirical observations, we do find that experiments with $\alpha$ near 0.5 perform very well. Intuitively, smaller $\alpha$ might decrease the SNR of estimators and result in bad prediction performance. On the other hand, bigger $\alpha$ might obscure meaningful critical points in the marginal likelihood function. A moderate $\alpha$ (close to 0.5) balances this dilemma and provide a promising result. Indeed, this argument can be further supported by Figure \ref{result:alpha}. We uniformly sample 20 $\alpha$ values ranging from 0.05 to 0.95 and plot the mean RMSE with respect to the corresponding $\alpha$. This plot demonstrates that intermediate $\alpha$ values are best on average.  

Lastly, the advantages of our model become increasingly significant when the sample size increases. This reveals that controlling smoothness and shape of ELBO is necessary and promising when we have big and high dimensional data.

We only report the optimal and some interesting $\alpha$ values in this section due to limited space. In the appendix, we provide more experimental results.


\section{Conclusion}
\label{sec:conclusion}
In this paper, we introduce an alternative closed form lower bound $\alpha$-ELBO on the $\mathcal{GP}$ likelihood based on the R\'enyi $\alpha$-divergence. This bound generalizes the exact and sparse $\mathcal{GP}$ likelihood and is capable of controlling and tuning regularization on the model inference. Our model has the same computation complexity as the exact $\mathcal{GP}$ and can be efficiently learned by the distributed BBMM algorithm. Throughout many numerical studies, we show that the proposed model may be able to deliver improvement over several $\mathcal{GP}$ inference framework. 

One future direction is to extend our model into non-Gaussian likelihood \cite{sheth2015sparse}. Another promising direction is to develop a framework on selecting the optimal tuning parameter. We hope our work spurs interest in the merits of using R\'enyi $\mathcal{GP}$ inference which allows the data to decide the degree of enforced regularization.

\clearpage

\onecolumn

\section*{Appendix}
This appendix contains all technical details in our main paper. In Sec. 8, we review some well-known properties of R\'enyi divergence. We provide a detailed derivation of the variational R\'enyi lower bound in Sec. 9. In Sec. 10, we provide proofs of our convergence results. These proofs are built on many lemmas and claims. In Sec. 11, we give more details about computation and parameter estimation. 

\section{Properties of R\'enyi Divergence}
\begin{claim}
    $\lim_{\alpha\to 1}D_\alpha[p||q]=KL[p||q]$.
\end{claim}
\begin{proof}
    Applying the L'Hopital rule, we have
    \begin{equation*}
        \begin{split}
            \lim_{\alpha\to 1}D_\alpha[p||q]&=\lim_{\alpha\to 1}\frac{1}{\alpha-1}\log\int p(\bm{\theta})^\alpha q(\bm{\theta})^{1-\alpha}d\bm{\theta}\\
            &=\lim_{\alpha\to 1}\frac{1}{\frac{d}{d\alpha}(\alpha-1)}\frac{d}{d\alpha}\log\int p(\bm{\theta})^\alpha q(\bm{\theta})^{1-\alpha}d\bm{\theta}\\
            &=\lim_{\alpha\to 1}\frac{d}{d\alpha}\log\int p(\bm{\theta})^\alpha q(\bm{\theta})^{1-\alpha}d\bm{\theta}.
        \end{split}
    \end{equation*}
    By the Leibniz's rule, we have
    \begin{equation*}
        \begin{split}
           \lim_{\alpha\to 1}D_\alpha[p||q]&=\lim_{\alpha\to 1}\frac{d}{d\alpha}\log\int p(\bm{\theta})^\alpha q(\bm{\theta})^{1-\alpha}d\bm{\theta}\\
           &=\lim_{\alpha\to 1}\frac{\int p(\bm{\theta})^\alpha q(\bm{\theta})^{1-\alpha}[\log p(\bm{\theta})-\log q(\bm{\theta})]d\bm{\theta}}{\int p(\bm{\theta})^\alpha q(\bm{\theta})^{1-\alpha}d\bm{\theta}}\\
           &=\frac{\int p(\bm{\theta})[\log p(\bm{\theta})-\log q(\bm{\theta})]d\bm{\theta}}{\int p(\bm{\theta}) d\bm{\theta}}\\
           &=\int p(\bm{\theta})\log\frac{p(\bm{\theta})}{q(\bm{\theta})}d\bm{\theta}=KL[p||q].
        \end{split}
    \end{equation*}
\end{proof}

\begin{claim}
    $\mathcal{L}_0(q;\bm{y})=\log P(\bm{y})$.
\end{claim}

\begin{claim}
    $\mathcal{L}_{VI}=\lim_{\alpha\to1}\mathcal{L}_\alpha(q;\bm{y})\leq\mathcal{L}_{\alpha_+}(q;\bm{y})\leq\log P(\bm{y})\leq \mathcal{L}_{\alpha-}(q;\bm{y}), \forall\alpha_+\in(0,1),\alpha_-<0$.
\end{claim}
\begin{proof}
    The first equality follows from the Claim 1. The left inequality can be obtained by the Jensen's inequality. Remaining inequalities are true using the fact that R\'enyi’s $\alpha$-divergence is continuous and non-decreasing on $\alpha\in\mathbb{R}$.
\end{proof}

\section{The Variational R\'enyi Lower Bound}
Let $q\coloneqq q(\bm{f},\bm{U}|\bm{\mathcal{Z}})$ and $p\coloneqq p(\bm{f},\bm{U},\bm{y}|\bm{\mathcal{Z}})$. When we apply the R\'enyi divergence to $\mathcal{GP}$ and assume that $q(\bm{f},\bm{U}|\bm{\mathcal{Z}})=p(\bm{f}|\bm{U},\bm{\mathcal{Z}})q(\bm{U})$, we can further obtain
\begin{equation*}
    \begin{split}
        \mathcal{L}_\alpha(q;\bm{y})&\coloneqq\frac{1}{1-\alpha}\log\mathbb{E}_q\bigg[\bigg(\frac{p(\bm{f},\bm{U},\bm{y}|\bm{\mathcal{Z}})}{q(\bm{f},\bm{U}|\bm{\mathcal{Z}})}\bigg)^{1-\alpha}\bigg]\\
        &=\frac{1}{1-\alpha}\log\mathbb{E}_q\bigg[\bigg(\frac{p(\bm{y}|\bm{f})\cancel{p(\bm{f}|\bm{U},\bm{\mathcal{Z}})}p(\bm{U}|\bm{\mathcal{Z}})}{\cancel{(\bm{f}|\bm{U},\bm{\mathcal{Z}})}q(\bm{U})}\bigg)^{1-\alpha}\bigg]\\
        &=\frac{1}{1-\alpha}\log\int p(\bm{f}|\bm{U},\bm{\mathcal{Z}})q(\bm{U})\bigg(\frac{p(\bm{y}|\bm{f})p(\bm{U}|\bm{\mathcal{Z}})}{q(\bm{U})}\bigg)^{1-\alpha} d\bm{U}d\bm{f}\\
        &=\frac{1}{1-\alpha}\log\int p(\bm{f}|\bm{U},\bm{\mathcal{Z}})q(\bm{U})^\alpha\bigg(p(\bm{y}|\bm{f})p(\bm{U}|\bm{\mathcal{Z}})\bigg)^{1-\alpha} d\bm{U}d\bm{f}\\
        &=\frac{1}{1-\alpha}\log\int p(\bm{f}|\bm{U},\bm{\mathcal{Z}})p(\bm{y}|\bm{f})^{1-\alpha}d\bm{f}\int q(\bm{U})^\alpha p(\bm{U}|\bm{\mathcal{Z}})^{1-\alpha} d\bm{U}.
    \end{split}
\end{equation*}

It can be easily shown that $p(\bm{f}|\bm{U},\bm{\mathcal{Z}})=\mathcal{N}(\bm{K}_{\bm{f},\bm{U}}\bm{K}_{\bm{U},\bm{U}}^{-1}\bm{U},\bm{K}_{\bm{f},\bm{f}}-\bm{Q})$, where $\bm{Q}=\bm{K}_{\bm{f},\bm{U}}\bm{K}_{\bm{U},\bm{U}}^{-1}\bm{K}_{\bm{U},\bm{f}}$. Besides, we have $p(\bm{y}|\bm{f})=\mathcal{N}(\bm{f},\sigma_\epsilon^2 I)$. Therefore,
\begin{equation*}
    \begin{split}
        &\int p(\bm{f}|\bm{U},\bm{\mathcal{Z}})p(\bm{y}|\bm{f})^{1-\alpha}d\bm{f}\\
        &=\int p(\bm{f}|\bm{U},\bm{\mathcal{Z}})(|2\pi\sigma^2_\epsilon I|^{-0.5}e^{-\frac{1}{2}(\bm{y}-\bm{f})^T(\sigma^2_\epsilon I)^{-1}(\bm{y}-\bm{f})})^{1-\alpha}d\bm{f}\\
        &=\frac{|2\pi\sigma_\epsilon^2I|^{-0.5(1-\alpha)}}{|2\pi\sigma_\epsilon^2I/(1-\alpha)|^{-0.5}}\int p(\bm{f}|\bm{U},\bm{\mathcal{Z}})\mathcal{N}(\bm{f},\frac{\sigma_\epsilon^2I}{1-\alpha}) d\bm{f}\\
        &=\frac{|2\pi\sigma_\epsilon^2I|^{-0.5(1-\alpha)}}{|2\pi\sigma_\epsilon^2I/(1-\alpha)|^{-0.5}}\mathcal{N}(\bm{K}_{\bm{f},\bm{U}}\bm{K}_{\bm{U},\bm{U}}^{-1}\bm{U},\frac{\sigma_\epsilon^2}{1-\alpha} I+\bm{K}_{\bm{f},\bm{f}}-\bm{Q})\\
        &=(2\pi\sigma_\epsilon^2)^{\frac{\alpha N}{2}}(\frac{1}{1-\alpha})^{\frac{N}{2 }}\mathcal{N}(\bm{K}_{\bm{f},\bm{U}}\bm{K}_{\bm{U},\bm{U}}^{-1}\bm{U},\frac{\sigma_\epsilon^2}{1-\alpha} I+\bm{K}_{\bm{f},\bm{f}}-\bm{Q})\\
        &=p(\bm{y}|\bm{U},\bm{\mathcal{Z}}).
    \end{split}
\end{equation*}
Instead of treating $q(\bm{U})$ as a pool of free parameters, it is desirable to find the optimal $q^*(\bm{U})$ to maximize the lower bound. This can be achieved by the special case of the H\"older inequality (i.e., Lyapunov inequality). Then we have,
\begin{equation*}
    \begin{split}
        \mathcal{L}_\alpha(q;\bm{y})&=\frac{1}{1-\alpha}\log\int p(\bm{y}|\bm{U},\bm{\mathcal{Z}}) q(\bm{U})^\alpha p(\bm{U}|\bm{\mathcal{Z}})^{1-\alpha} d\bm{U}\\
        &=\frac{1}{1-\alpha}\log\int q(\bm{U})(\frac{p(\bm{y}|\bm{U},\bm{\mathcal{Z}})^{1/(1-\alpha)}p(\bm{U}|\bm{\mathcal{Z}})}{q(\bm{U})})^{1-\alpha} d\bm{U}\\
        &=\frac{1}{1-\alpha}\log\mathbb{E}_q(\frac{p(\bm{y}|\bm{U},\bm{\mathcal{Z}})^{1/(1-\alpha)}p(\bm{U}|\bm{\mathcal{Z}})}{q(\bm{U})})^{1-\alpha}\\
        &\leq \frac{1}{1-\alpha}\log[\mathbb{E}_q(\frac{p(\bm{y}|\bm{U},\bm{\mathcal{Z}})^{1/(1-\alpha)}p(\bm{U}|\bm{\mathcal{Z}})}{q(\bm{U})})]^{1-\alpha}\\
        &=\log\mathbb{E}_q(\frac{p(\bm{y}|\bm{U},\bm{\mathcal{Z}})^{1/(1-\alpha)}p(\bm{U}|\bm{\mathcal{Z}})}{q(\bm{U})})\\
        &=\log\int p(\bm{y}|\bm{U},\bm{\mathcal{Z}})^{1/(1-\alpha)}p(\bm{U}|\bm{\mathcal{Z}}) d\bm{U}.
    \end{split}
\end{equation*}

The optimal $q(\bm{U})$ is 
\begin{equation*}
    q^*(\bm{U})\propto p(\bm{y}|\bm{U},\bm{\mathcal{Z}})^{1/(1-\alpha)}p(\bm{U}|\bm{\mathcal{Z}}).
\end{equation*}

Specifically,
\begin{equation*}
    q^*(\bm{U})=\frac{p(\bm{y}|\bm{U},\bm{\mathcal{Z}})^{1/(1-\alpha)}p(\bm{U}|\bm{\mathcal{Z}})}{\int p(\bm{y}|\bm{U},\bm{\mathcal{Z}})^{1/(1-\alpha)}p(\bm{U}|\bm{\mathcal{Z}}) d\bm{U}}.
\end{equation*}

It can be shown that
\begin{equation*}
    \begin{split}
        p(\bm{y}|\bm{U},\bm{\mathcal{Z}})^{\frac{1}{1-\alpha}}&=[(2\pi\sigma_\epsilon^2)^{\frac{\alpha N}{2}}(\frac{1}{1-\alpha})^{\frac{N}{2 }}]^{\frac{{1}}{1-\alpha}}\mathcal{N}(\bm{K}_{\bm{f},\bm{U}}\bm{K}_{\bm{U},\bm{U}}^{-1}\bm{U},\frac{\sigma_\epsilon^2}{1-\alpha} I+\bm{K}_{\bm{f},\bm{f}}-\bm{Q})^{\frac{1}{1-\alpha}}\\
        &=[(2\pi\sigma_\epsilon^2)^{\frac{\alpha N}{2(1-\alpha)}}(\frac{1}{1-\alpha})^{\frac{N}{2(1-\alpha)}}] C \mathcal{N}(\bm{K}_{\bm{f},\bm{U}}\bm{K}_{\bm{U},\bm{U}}^{-1}\bm{U}, \sigma_\epsilon^2 I+(1-\alpha)[\bm{K}_{\bm{f},\bm{f}}-\bm{Q}]),
    \end{split}
\end{equation*}
where $C=\frac{|2\pi(\frac{\sigma_\epsilon^2}{1-\alpha} I+\bm{K}_{\bm{f},\bm{f}}-\bm{Q})|^{-0.5/(1-\alpha)}}{|2\pi(\sigma_\epsilon^2 I+(1-\alpha)[\bm{K}_{\bm{f},\bm{f}}-\bm{Q}])|^{-0.5}}=|2\pi(\frac{\sigma_\epsilon^2}{1-\alpha} I+\bm{K}_{\bm{f},\bm{f}}-\bm{Q})|^{\frac{-\alpha}{2(1-\alpha)}}(1-\alpha)^{N/2}$. Since $p(\bm{U}|\bm{\mathcal{Z}})=\mathcal{N}(\bm{0},\bm{K}_{\bm{U},\bm{U}})$, we have
\begin{equation*}
    \begin{split}
        \mathcal{L}_\alpha(q;\bm{y})&=\log\int p(\bm{y}|\bm{U},\bm{\mathcal{Z}})^{1/(1-\alpha)}p(\bm{U}|\bm{\mathcal{Z}}) d\bm{U}\\
        &=\log C_x\mathcal{N}(\bm{0}, \sigma_\epsilon^2 I+(1-\alpha)[\bm{K}_{\bm{f},\bm{f}}-\bm{Q}]+\bm{K}_{\bm{f},\bm{U}}\bm{K}_{\bm{U},\bm{U}}^{-1}\bm{K}_{\bm{U},\bm{f}})\\
        &=\log C_x\mathcal{N}(\bm{0}, \sigma_\epsilon^2 I+(1-\alpha)[\bm{K}_{\bm{f},\bm{f}}-\bm{Q}]+\bm{Q})\\
        &=\log\mathcal{N}(\bm{0}, \sigma_\epsilon^2 I+(1-\alpha)[\bm{K}_{\bm{f},\bm{f}}]+\alpha\bm{Q}) + \log C_x,
    \end{split}
\end{equation*}
where
\begin{equation*}
    \begin{split}
        C_x&=[(2\pi\sigma_\epsilon^2)^{\frac{\alpha N}{2(1-\alpha)}}(\frac{1}{1-\alpha})^{\frac{N}{2(1-\alpha)}}][|2\pi(\frac{\sigma_\epsilon^2}{1-\alpha} I+\bm{K}_{\bm{f},\bm{f}}-\bm{Q})|^{\frac{-\alpha}{2(1-\alpha)}}(1-\alpha)^{N/2}]\\
        &=(2\pi\sigma_\epsilon^2)^{\frac{\alpha N}{2(1-\alpha)}}(1-\alpha)^{\frac{-\alpha N}{2(1-\alpha)}}|2\pi(\frac{\sigma_\epsilon^2}{1-\alpha} I+\bm{K}_{\bm{f},\bm{f}}-\bm{Q})|^{\frac{-\alpha}{2(1-\alpha)}}\\
        &=|\bm{I}+\frac{1-\alpha}{\sigma_\epsilon^2}(\bm{K}_{\bm{f},\bm{f}}-\bm{Q})|^{\frac{-\alpha}{2(1-\alpha)}}\\
        &\approx\bigg\{1+\frac{1-\alpha}{\sigma_\epsilon^2}\text{Tr}(\bm{K}_{\bm{f},\bm{f}}-\bm{Q})+\mathcal{O}(\frac{(1-\alpha)^2}{\sigma_\epsilon^4})\bigg\}^{\frac{-\alpha}{2(1-\alpha)}}.
    \end{split}
\end{equation*}
The last equality comes from the variation of Jacobi's formula. The $\approx$ approximates well only when $\frac{1-\alpha}{\sigma_\epsilon^2}$ is ``small". Therefore, the lower bound can be expressed as
\begin{equation*}
    \begin{split}
        \mathcal{L}_\alpha(q;\bm{y})\\
        &\approx\log\mathcal{N}(\bm{0}, \sigma_\epsilon^2 I+(1-\alpha)[\bm{K}_{\bm{f},\bm{f}}]+\alpha\bm{Q}) + \log \bigg\{1+\frac{1-\alpha}{\sigma_\epsilon^2}\text{Tr}(\bm{K}_{\bm{f},\bm{f}}-\bm{Q})+\mathcal{O}(\frac{(1-\alpha)^2}{\sigma_\epsilon^4})\bigg\}^{\frac{-\alpha}{2(1-\alpha)}},
    \end{split}
\end{equation*}
given that $\alpha\approx 1$. While this form is attractive, it is not practically useful since when $1-\alpha$ is ``large", the approximation does not work well. In the analysis section, we will instead use $|\bm{I}+\frac{1-\alpha}{\sigma_\epsilon^2}(\bm{K}_{\bm{f},\bm{f}}-\bm{Q})|^{\frac{-\alpha}{2(1-\alpha)}}$ to prove the convergence result.

\section{Convergence Results and Risk Bound}
\begin{lemma}
Suppose we have two positive semi-definite (PSD) matrices $A$ and $B$ such that $A-B$ is also a PSD matrix, then $|A|\geq|B|$. Furthermore, if $A$ and $B$ are positive definite (PD), then $B^{-1}\geq A^{-1}$.
\end{lemma}
This lemma has been proved in \citep{horn2012matrix}. Based on this lemma, we can compute a data-dependent upper bound on the log-marginal likelihood \citep{titsias2014variational}.
\begin{claim}
\label{upper}
$\log p(\bm{y})\leq\log\frac{1}{|2\pi((1-\alpha)\bm{K}_{\bm{f},\bm{f}}+\alpha\bm{Q}+\sigma_\epsilon^2\bm{I})|^{\frac{1}{2}}} e^{-\frac{1}{2}\bm{y}^T((1-\alpha)\bm{K}_{\bm{f},\bm{f}}+\alpha\bm{Q}+\alpha\text{Tr}(\bm{K}_{\bm{f},\bm{f}}-\bm{Q})\bm{I}+\sigma_\epsilon^2\bm{I})^{-1}\bm{y}}\coloneqq\mathcal{L}_{upper}$.
\end{claim}

\begin{proof}
Since
\begin{equation*}
    \bm{K}_{\bm{f},\bm{f}}+\sigma_\epsilon^2\bm{I}=(1-\alpha)\bm{K}_{\bm{f},\bm{f}}+\alpha\bm{K}_{\bm{f},\bm{f}}+\sigma_\epsilon^2\bm{I}\succeq (1-\alpha)\bm{K}_{\bm{f},\bm{f}}+\alpha\bm{Q}+\sigma_\epsilon^2\bm{I}\succeq 0,
\end{equation*}
where $\bm{A}\succeq\bm{B}$ means $\bm{x}^T\bm{A}\bm{x}\geq\bm{x}^T\bm{B}\bm{x}\geq\bm{0}, \forall\bm{x}$. 
Then, we can obtain $|\bm{K}_{\bm{f},\bm{f}}+\sigma_\epsilon^2\bm{I}|\geq|(1-\alpha)\bm{K}_{\bm{f},\bm{f}}+\alpha\bm{Q}+\sigma_\epsilon^2\bm{I}|$ since they are both PSD matrix. Therefore,
\begin{equation*}
    \frac{1}{|2\pi(\bm{K}_{\bm{f},\bm{f}}+\sigma_\epsilon^2\bm{I})|^{\frac{1}{2}}}\leq\frac{1}{|2\pi((1-\alpha)\bm{K}_{\bm{f},\bm{f}}+\alpha\bm{Q}+\sigma_\epsilon^2\bm{I})|^{\frac{1}{2}}}.
\end{equation*}
Let $\bm{U}\bm{\Lambda}\bm{U}^T$ be the eigen-decomposition of $\bm{K}_{\bm{f},\bm{f}}-\bm{Q}$. This decomposition exists since the matrix is PD. Then
\begin{equation*}
    \begin{split}
        \bm{y}^T\bm{U}\bm{\Lambda}\bm{U}^T\bm{y}&=\bm{z}^T\bm{\Lambda}\bm{z}=\sum_{i=1}^N\lambda_iz_i^2\leq\lambda_{max}\sum_{i=1}^Nz_i^2=\lambda_{max}\norm{\bm{z}}^2\\
        &=\lambda_{max}\norm{\bm{y}}^2\leq\sum_{i=1}^N\lambda_i\norm{\bm{y}}^2\leq\text{Tr}(\bm{K}_{\bm{f},\bm{f}}-\bm{Q})\norm{\bm{y}}^2,
    \end{split}
\end{equation*}
where $\bm{z}=\bm{U}^T\bm{y}$, $\{\lambda_i\}_{i=1}^N$ are eigenvalues of $\bm{K}_{\bm{f},\bm{f}}-\bm{Q}$ and $\lambda_{max}=\max(\lambda_1,\ldots,\lambda_N)$. Therefore, we have $\bm{y}^T(\bm{K}_{\bm{f},\bm{f}}-\bm{Q})\bm{y}\leq\text{Tr}(\bm{K}_{\bm{f},\bm{f}}-\bm{Q})\norm{\bm{y}}^2=\text{Tr}(\bm{K}_{\bm{f},\bm{f}}-\bm{Q})\bm{y}^T\bm{y}$. Apparently, $\alpha\bm{y}^T(\bm{K}_{\bm{f},\bm{f}}-\bm{Q})\bm{y}\leq\alpha\text{Tr}(\bm{K}_{\bm{f},\bm{f}}-\bm{Q})\bm{y}^T\bm{y}$. Therefore, we can obtain
\begin{equation*}
    \begin{split}
        &\bm{y}^T(\bm{K}_{\bm{f},\bm{f}}+\sigma_\epsilon^2\bm{I})\bm{y}\leq \bm{y}^T((1-\alpha)\bm{K}_{\bm{f},\bm{f}}+\alpha\bm{Q}+\sigma_\epsilon^2\bm{I})\bm{y}+\alpha\text{Tr}(\bm{K}_{\bm{f},\bm{f}}-\bm{Q})\bm{y}^T\bm{y}\\
        &=\bm{y}^T((1-\alpha)\bm{K}_{\bm{f},\bm{f}}+\alpha\bm{Q}+\alpha\text{Tr}(\bm{K}_{\bm{f},\bm{f}}-\bm{Q})\bm{I}+\sigma_\epsilon^2\bm{I})\bm{y}.
    \end{split}    
\end{equation*}
Based on this inequality, it is easy to show that
\begin{equation*}
    e^{-\frac{1}{2}\bm{y}^T(\bm{K}_{\bm{f},\bm{f}}+\sigma_\epsilon^2\bm{I})^{-1}\bm{y}}\leq e^{-\frac{1}{2}\bm{y}^T((1-\alpha)\bm{K}_{\bm{f},\bm{f}}+\alpha\bm{Q}+\alpha\text{Tr}(\bm{K}_{\bm{f},\bm{f}}-\bm{Q})\bm{I}+\sigma_\epsilon^2\bm{I})^{-1}\bm{y}}.
\end{equation*}
Finally, we obtain
\begin{equation*}
    \begin{split}
            &\frac{1}{|2\pi(\bm{K}_{\bm{f},\bm{f}}+\sigma_\epsilon^2\bm{I})|^{\frac{1}{2}}}e^{-\frac{1}{2}\bm{y}^T(\bm{K}_{\bm{f},\bm{f}}+\sigma_\epsilon^2\bm{I})^{-1}\bm{y}}\\
            &\leq\frac{1}{|2\pi((1-\alpha)\bm{K}_{\bm{f},\bm{f}}+\alpha\bm{Q}+\sigma_\epsilon^2\bm{I})|^{\frac{1}{2}}} e^{-\frac{1}{2}\bm{y}^T((1-\alpha)\bm{K}_{\bm{f},\bm{f}}+\alpha\bm{Q}+\alpha\text{Tr}(\bm{K}_{\bm{f},\bm{f}}-\bm{Q})\bm{I}+\sigma_\epsilon^2\bm{I})^{-1}\bm{y}}.
    \end{split}
\end{equation*}
\end{proof}
We will use this upper bound to prove our main theorem. 
\begin{claim}
    \label{detbound}
    $-\log|\bm{I}+\frac{1-\alpha}{\sigma_\epsilon^2}(\bm{K}_{\bm{f},\bm{f}}-\bm{Q})|^{\frac{-\alpha}{2(1-\alpha)}} \leq \frac{\alpha}{2(1-\alpha)}\log \bigg(\frac{\text{Tr}(\bm{I}+\frac{1-\alpha}{\sigma_\epsilon^2}(\bm{K}_{\bm{f},\bm{f}}-\bm{Q}))}{N}\bigg)^N$.
\end{claim}
\begin{proof}
Based on the inequality of arithmetic and geometric means, we have
\begin{equation*}
    \frac{\text{Tr}(M)}{N}\geq |M|^{1/N},
\end{equation*}
given an positive semi-definite matrix $M$ with dimension $N$. Therefore, we can obtain 
\begin{equation*}
    |\bm{I}+\frac{1-\alpha}{\sigma_\epsilon^2}(\bm{K}_{\bm{f},\bm{f}}-\bm{Q})|^{1/N}\leq\frac{\text{Tr}(\bm{I}+\frac{1-\alpha}{\sigma_\epsilon^2}(\bm{K}_{\bm{f},\bm{f}}-\bm{Q}))}{N}.
\end{equation*}
By some simple algebra manipulation, we will obtain
\begin{equation*}
    \frac{\alpha}{2(1-\alpha)}\log |\bm{I}+\frac{1-\alpha}{\sigma_\epsilon^2}(\bm{K}_{\bm{f},\bm{f}}-\bm{Q})|\leq \frac{\alpha}{2(1-\alpha)}\log \bigg(\frac{\text{Tr}(\bm{I}+\frac{1-\alpha}{\sigma_\epsilon^2}(\bm{K}_{\bm{f},\bm{f}}-\bm{Q}))}{N}\bigg)^N.
\end{equation*}
\end{proof}

We first provide a lower bound and an upper bound on the R\'enyi divergence.
\begin{lemma}
For any set of $\{\bm{x}_i\}_{i=1}^N$, if the output $\{y_i\}_{i=1}^N$ are generated according to some generative model, then
\begin{equation}
\label{lubound}
    \begin{split}
        &-\log |\bm{I}+\frac{1-\alpha}{\sigma_\epsilon^2}(\bm{K}_{\bm{f},\bm{f}}-\bm{Q})|^{\frac{-\alpha}{2(1-\alpha)}} \leq \mathbb{E}_y\bigg[D_\alpha[p||q]\bigg]\\
        &\leq -\log |\bm{I}+\frac{1-\alpha}{\sigma_\epsilon^2}(\bm{K}_{\bm{f},\bm{f}}-\bm{Q})|^{\frac{-\alpha}{2(1-\alpha)}} + \frac{\alpha\text{Tr}(\bm{K}_{\bm{f},\bm{f}}-\bm{Q})}{2\sigma_\epsilon^2}.
    \end{split}
\end{equation}
\end{lemma}
\begin{proof}
We have
\begin{equation*}
    \begin{split}
        &\mathbb{E}_y\bigg[D_\alpha[p||q]\bigg]\\
        &=\mathbb{E}_y\bigg[\log p(\bm{y})-\log\mathcal{N}(\bm{0}, \sigma_\epsilon^2 I+(1-\alpha)\bm{K}_{\bm{f},\bm{f}}+\alpha\bm{Q}) -\log|\bm{I}+\frac{1-\alpha}{\sigma_\epsilon^2}(\bm{K}_{\bm{f},\bm{f}}-\bm{Q})|^{\frac{-\alpha}{2(1-\alpha)}}\bigg]\\
        &=-\log|\bm{I}+\frac{1-\alpha}{\sigma_\epsilon^2}(\bm{K}_{\bm{f},\bm{f}}-\bm{Q})|^{\frac{-\alpha}{2(1-\alpha)}} + \mathbb{E}_y\bigg[\log\frac{\mathcal{N}(\bm{0},\bm{K}_{\bm{f},\bm{f}}+\sigma_\epsilon^2\bm{I})}{\mathcal{N}(\bm{0}, \sigma_\epsilon^2 \bm{I}+(1-\alpha)\bm{K}_{\bm{f},\bm{f}}+\alpha\bm{Q})}\bigg].
    \end{split}
\end{equation*}
It is apparent that the lower bound to \eqref{lubound} is
\begin{equation*}
    -\log|\bm{I}+\frac{1-\alpha}{\sigma_\epsilon^2}(\bm{K}_{\bm{f},\bm{f}}-\bm{Q})|^{\frac{-\alpha}{2(1-\alpha)}},
\end{equation*}
since the KL divergence is non-negative. We then provide an upper bound to \eqref{lubound}. We have
\begin{equation*}
    \begin{split}
        &-\log|\bm{I}+\frac{1-\alpha}{\sigma_\epsilon^2}(\bm{K}_{\bm{f},\bm{f}}-\bm{Q})|^{\frac{-\alpha}{2(1-\alpha)}} + \mathbb{E}_y\bigg[\log\frac{\mathcal{N}(\bm{0},\bm{K}_{\bm{f},\bm{f}}+\sigma_\epsilon^2\bm{I})}{\mathcal{N}(\bm{0}, \sigma_\epsilon^2 \bm{I}+(1-\alpha)\bm{K}_{\bm{f},\bm{f}}+\alpha\bm{Q})}\bigg]\\
        &=-\log|\bm{I}+\frac{1-\alpha}{\sigma_\epsilon^2}(\bm{K}_{\bm{f},\bm{f}}-\bm{Q})|^{\frac{-\alpha}{2(1-\alpha)}}\\
        &\quad -\frac{N}{2}+\frac{1}{2}\log\bigg(\frac{|\sigma_\epsilon^2 \bm{I}+(1-\alpha)\bm{K}_{\bm{f},\bm{f}}+\alpha\bm{Q}|}{|\bm{K}_{\bm{f},\bm{f}}+\sigma_\epsilon^2\bm{I}|}\bigg) + \frac{1}{2}\text{Tr}\big((\sigma_\epsilon^2 \bm{I}+(1-\alpha)\bm{K}_{\bm{f},\bm{f}}+\alpha\bm{Q})^{-1}(\bm{K}_{\bm{f},\bm{f}}+\sigma_\epsilon^2\bm{I})\big)\\
        &\leq -\log|\bm{I}+\frac{1-\alpha}{\sigma_\epsilon^2}(\bm{K}_{\bm{f},\bm{f}}-\bm{Q})|^{\frac{-\alpha}{2(1-\alpha)}}-\frac{N}{2} + \frac{1}{2}\text{Tr}\big((\sigma_\epsilon^2 \bm{I}+(1-\alpha)\bm{K}_{\bm{f},\bm{f}}+\alpha\bm{Q})^{-1}(\bm{K}_{\bm{f},\bm{f}}+\sigma_\epsilon^2\bm{I})\big).
    \end{split}
\end{equation*}
This inequality follows from the fact that $\bm{K}_{\bm{f},\bm{f}}+\sigma_\epsilon^2\bm{I}\succeq\sigma_\epsilon^2 \bm{I}+(1-\alpha)\bm{K}_{\bm{f},\bm{f}}+\alpha\bm{Q}$. Since 
\begin{equation*}
    \begin{split}
        &\frac{1}{2}\text{Tr}\big((\sigma_\epsilon^2 \bm{I}+(1-\alpha)\bm{K}_{\bm{f},\bm{f}}+\alpha\bm{Q})^{-1}(\bm{K}_{\bm{f},\bm{f}}+\sigma_\epsilon^2\bm{I})\big)\\
        &=\frac{1}{2}\text{Tr}(\bm{I})+\frac{1}{2}\text{Tr}\bigg((\sigma_\epsilon^2 \bm{I}+(1-\alpha)\bm{K}_{\bm{f},\bm{f}}+\alpha\bm{Q})^{-1}(\bm{\tilde K})\bigg)\\
        &\leq \frac{N}{2}+\alpha\text{Tr}(\bm{K}_{\bm{f},\bm{f}}-\bm{Q})\lambda_1((\sigma_\epsilon^2 \bm{I}+(1-\alpha)\bm{K}_{\bm{f},\bm{f}}+\alpha\bm{Q})^{-1})/2\\
        &\leq \frac{N}{2}+\frac{\alpha\text{Tr}(\bm{K}_{\bm{f},\bm{f}}-\bm{Q})}{2\sigma_\epsilon^2},
    \end{split}
\end{equation*}
where $\bm{\tilde K}=\bm{K}_{\bm{f},\bm{f}}+\sigma_\epsilon^2\bm{I}-\big(\sigma_\epsilon^2 \bm{I}+(1-\alpha)\bm{K}_{\bm{f},\bm{f}}+\alpha\bm{Q}\big)$ and $\lambda_1(\bm{M})$ is the largest eigenvalue of an arbitrary matrix $M$. We apply the H\"older's inequality for schatten norms to the second last inequality. Therefore, we obtain the upper bound as follow.
\begin{equation*}
    -\log|\bm{I}+\frac{1-\alpha}{\sigma_\epsilon^2}(\bm{K}_{\bm{f},\bm{f}}-\bm{Q})|^{\frac{-\alpha}{2(1-\alpha)}} + \frac{\alpha\text{Tr}(\bm{K}_{\bm{f},\bm{f}}-\bm{Q})}{2\sigma_\epsilon^2}.
\end{equation*}
\end{proof}
As $\alpha\to 1$, we recover the bounds for the KL divergence. Specifically, we get the lower bound $\frac{\text{Tr}(\bm{K}_{\bm{f},\bm{f}}-\bm{Q})}{2\sigma_\epsilon^2}$ and upper bound $\frac{\text{Tr}(\bm{K}_{\bm{f},\bm{f}}-\bm{Q})}{\sigma_\epsilon^2}$ \citep{burt2019rates}.

\begin{lemma}
Given a symmetric positive semidefinite matrix $\bm{K}_{\bm{f},\bm{f}}$, if $M$ columns are selected
to form a Nystr\"om approximation such that the probability of selecting a subset of columns $Z$ is proportional to the
determinant of the principal submatrix formed by these
columns and the matching rows, then
\begin{equation*}
    \mathbb{E}_Z\bigg[\text{Tr}(\bm{K}_{\bm{f},\bm{f}}-\bm{Q})\bigg]\leq(M+1)\sum_{m=M+1}^N\lambda_m(\bm{K}_{\bm{f},\bm{f}}).
\end{equation*}
\end{lemma}
This lemma is proved in \citep{belabbas2009spectral}. Following this lemma and by Lemma \ref{detbound}, we can show that
\begin{equation*}
    \begin{split}
        &\mathbb{E}_Z\bigg[-\log |\bm{I}+\frac{1-\alpha}{\sigma_\epsilon^2}(\bm{K}_{\bm{f},\bm{f}}-\bm{Q})|^{\frac{-\alpha}{2(1-\alpha)}} \bigg]\\
        &=\mathbb{E}_Z\bigg[\frac{\alpha}{2(1-\alpha)}\log |\bm{I}+\frac{1-\alpha}{\sigma_\epsilon^2}(\bm{K}_{\bm{f},\bm{f}}-\bm{Q})| \bigg]\\
        &\leq \mathbb{E}_Z\bigg[\frac{\alpha}{2(1-\alpha)}\log\bigg(\frac{\text{Tr}(\bm{I}+\frac{1-\alpha}{\sigma_\epsilon^2}(\bm{K}_{\bm{f},\bm{f}}-\bm{Q}))}{N}\bigg)^N \bigg] \\
        &\leq \frac{\alpha N}{2(1-\alpha)}\log\mathbb{E}_Z\bigg[\bigg(\frac{\text{Tr}(\bm{I}+\frac{1-\alpha}{\sigma_\epsilon^2}(\bm{K}_{\bm{f},\bm{f}}-\bm{Q}))}{N}\bigg) \bigg]\\
        &\leq \frac{\alpha N}{2(1-\alpha)}\log \bigg\{1+\frac{1-\alpha}{\sigma_\epsilon^2}\frac{(M+1)\sum_{m=M+1}^N\lambda_m(\bm{K}_{\bm{f},\bm{f}})}{N} \bigg\}.
    \end{split}
\end{equation*}
As $\alpha\to 1$, this bound becomes $\frac{1}{2\sigma_\epsilon^2}(M+1)\sum_{m=M+1}^N\lambda_m(\bm{K}_{\bm{f},\bm{f}})$. Following the inequality and lemma above, we can obtain the following corollary.
\begin{corollary}
\begin{equation*}
    \mathbb{E}_{Z\sim v}[\text{Tr}(\bm{K}_{\bm{f},\bm{f}}-\bm{Q})]\leq(M+1)\sum_{m=M+1}^N\lambda_m(\bm{K}_{\bm{f},\bm{f}})+2Nv\epsilon.
\end{equation*}
\end{corollary}
This inequality is from \citep{burt2019rates}. Using this fact, we can show that
\begin{equation*}
    \begin{split}
        &\mathbb{E}_{Z\sim v}\bigg[-\log |\bm{I}+\frac{1-\alpha}{\sigma_\epsilon^2}(\bm{K}_{\bm{f},\bm{f}}-\bm{Q})|^{\frac{-\alpha}{2(1-\alpha)}} \bigg]\\
        &\leq\frac{\alpha}{2(1-\alpha)}\log\mathbb{E}_{Z\sim v}\bigg[ \log\bigg(\frac{\text{Tr}(\bm{I}+\frac{1-\alpha}{\sigma_\epsilon^2}(\bm{K}_{\bm{f},\bm{f}}-\bm{Q}))}{N}\bigg)^N\bigg]\\
        &\leq\frac{\alpha N}{2(1-\alpha)}\log\bigg[1+\frac{1-\alpha}{\sigma_\epsilon^2}\frac{[(M+1)\sum_{m=M+1}^N\lambda_m(\bm{K}_{\bm{f},\bm{f}})+2Nv\epsilon]}{N}\bigg].
    \end{split}
\end{equation*}

The next theorem is based on a lemma. We will prove this lemma first.
\begin{lemma}
\label{posbound}
Then,
\begin{equation*}
    \begin{split}
        D_\alpha[p||q]&\leq-\log |\bm{I}+\frac{1-\alpha}{\sigma_\epsilon^2}(\bm{K}_{\bm{f},\bm{f}}-\bm{Q})|^{\frac{-\alpha}{2(1-\alpha)}}  + \norm{\bm{y}}^2\frac{\alpha\text{Tr}(\bm{K}_{\bm{f},\bm{f}}-\bm{Q})}{\sigma_{\epsilon}^4+\alpha\sigma_{\epsilon}^2\text{Tr}(\bm{K}_{\bm{f},\bm{f}}-\bm{Q})}
    \end{split}
\end{equation*}
where $\tilde\lambda_{max}$ is the largest eigenvalue of  $\bm{K}_{\bm{f},\bm{f}}-\bm{Q}$.
\end{lemma}
\begin{proof}
Based on Claim \ref{upper}, we have
\begin{equation*}
    \begin{split}
        \mathcal{L}_{upper}&=\log\frac{1}{|2\pi((1-\alpha)\bm{K}_{\bm{f},\bm{f}}+\alpha\bm{Q}+\sigma_\epsilon^2\bm{I})|^{\frac{1}{2}}} e^{-\frac{1}{2}\bm{y}^T((1-\alpha)\bm{K}_{\bm{f},\bm{f}}+\alpha\bm{Q}+\alpha\text{Tr}(\bm{K}_{\bm{f},\bm{f}}-\bm{Q})\bm{I}+\sigma_\epsilon^2\bm{I})^{-1}\bm{y}}  \\
        &\leq -\frac{1}{2}\log|(1-\alpha)\bm{K}_{\bm{f},\bm{f}}+\alpha\bm{Q}+\sigma_\epsilon^2\bm{I}|-\frac{N}{2}\log(2\pi)-\frac{1}{2}\bm{y}^T((1-\alpha)\bm{K}_{\bm{f},\bm{f}}+\alpha\bm{Q}+\alpha\tilde\lambda_{max}\bm{I}+\sigma_\epsilon^2\bm{I})^{-1}\bm{y} \\
        &\coloneqq\mathcal{L}'_{upper},
    \end{split}
\end{equation*}
using the fact that $\text{Tr}(\bm{K}_{\bm{f},\bm{f}}-\bm{Q})\geq\tilde\lambda_{max}$. Then, we have
\begin{equation*}
    \begin{split}
        &\mathcal{L}'_{upper}-\mathcal{L}_\alpha(q;\bm{y})\\
        &=-\log |\bm{I}+\frac{1-\alpha}{\sigma_\epsilon^2}(\bm{K}_{\bm{f},\bm{f}}-\bm{Q})|^{\frac{-\alpha}{2(1-\alpha)}} \\
        &\quad +\frac{1}{2}\bm{y}^T\bigg(((1-\alpha)\bm{K}_{\bm{f},\bm{f}}+\alpha\bm{Q}+\sigma_\epsilon^2\bm{I})^{-1} - ((1-\alpha)\bm{K}_{\bm{f},\bm{f}}+\alpha\bm{Q}+\alpha\tilde\lambda_{max}\bm{I}+\sigma_\epsilon^2\bm{I})^{-1}\bigg)\bm{y}.
    \end{split}
\end{equation*}
Let $(1-\alpha)\bm{K}_{\bm{f},\bm{f}}+\alpha\bm{Q}+\sigma_\epsilon^2\bm{I}=\bm{V}\bm{\Lambda_\alpha}\bm{V}^T$ be the eigenvalue decomposition and denote by $\gamma_1\geq\ldots\geq\gamma_N$ all eigenvalues. Then we can obtain
\begin{equation*}
    \begin{split}
        &\frac{1}{2}(\bm{V}^T\bm{y})^T\bigg(\bm{\Lambda_\alpha}^{-1}-(\bm{\Lambda_\alpha}+\alpha\tilde\lambda_{max}\bm{I})^{-1} \bigg)(\bm{V}^T\bm{y})\\
        &=\frac{1}{2}\bm{z'}^T\bigg(\bm{\Lambda_\alpha}^{-1}-(\bm{\Lambda_\alpha}+\alpha\tilde\lambda_{max}\bm{I})^{-1}\bigg)\bm{z'}\\
        &=\frac{1}{2}\sum_i z_i'^2\frac{\alpha\tilde\lambda_{max}}{\gamma_i^2+\alpha\gamma_i\tilde\lambda_{max}}\\
        &\leq\frac{1}{2}\norm{\bm{y}}^2\frac{\alpha\tilde\lambda_{max}}{\gamma_N^2+\alpha\gamma_N\tilde\lambda_{max}}\\
        &\leq\frac{1}{2}\norm{\bm{y}}^2\frac{\alpha\tilde\lambda_{max}}{\sigma_{\epsilon}^4+\alpha\sigma_{\epsilon}^2\tilde\lambda_{max}} ,
    \end{split}
\end{equation*}
where $\bm{z'}=\bm{V}^T\bm{y}$. Therefore, we have 
\begin{equation*}
    \begin{split}
         D_\alpha[p||q]&\leq-\log |\bm{I}+\frac{1-\alpha}{\sigma_\epsilon^2}(\bm{K}_{\bm{f},\bm{f}}-\bm{Q})|^{\frac{-\alpha}{2(1-\alpha)}}  + \frac{1}{2}\norm{\bm{y}}^2\frac{\alpha\tilde\lambda_{max}}{\sigma_{\epsilon}^4+\alpha\sigma_{\epsilon}^2\tilde\lambda_{max}}\\
         &\leq-\log |\bm{I}+\frac{1-\alpha}{\sigma_\epsilon^2}(\bm{K}_{\bm{f},\bm{f}}-\bm{Q})|^{\frac{-\alpha}{2(1-\alpha)}}  + \frac{1}{2}\norm{\bm{y}}^2\frac{\alpha\text{Tr}(\bm{K}_{\bm{f},\bm{f}}-\bm{Q})}{\sigma_{\epsilon}^4+\alpha\sigma_{\epsilon}^2\text{Tr}(\bm{K}_{\bm{f},\bm{f}}-\bm{Q})}.
    \end{split}
\end{equation*}
\end{proof}

\begin{theorem}
\label{bound1}
Suppose $N$ data points are drawn i.i.d from input distribution $p(\bm{x})$ and $k(\bm{x},\bm{x})\leq v, \forall \bm{x}\in\mathcal{X}$. Sample $M$ inducing points from the training data with the probability assigned to any set of size $M$ equal to the probability assigned to the corresponding subset by an $\epsilon$ k-Determinantal Point Process (k-DPP) \citep{belabbas2009spectral} with $k=M$. If $\bm{y}$ is distributed according to a sample from the prior generative model, with probability at least $1-\delta$,
\begin{equation*}
    \begin{split}
        D_\alpha[p||q]&\leq\alpha\frac{(M+1)N\sum_{m=M+1}^\infty\lambda_m+2Nv\epsilon}{2\delta\sigma_\epsilon^2}+\\
        &\frac{1}{\delta}\frac{\alpha }{2(1-\alpha)}\log\bigg[1+\frac{1-\alpha}{\sigma_\epsilon^2}\frac{[(M+1)N\sum_{m=M+1}^\infty\lambda_m+2Nv\epsilon]}{N}\bigg]^N.
    \end{split}
\end{equation*}
where $\lambda_m$ are the eigenvalues of the integral operator $\mathcal{K}$ associated to kernel, $k$ and $p(\bm{x})$.
\end{theorem}
\begin{proof}
We have
\begin{equation*}
    \begin{split}
        &\mathbb{E}_{\bm{X}}\bigg[\mathbb{E}_{Z|\bm{X}}\bigg[\mathbb{E}_{\bm{y}}\bigg[ D_\alpha[p||q]\bigg]\bigg]\bigg]\\
        &\leq\mathbb{E}_{\bm{X}}\bigg[\mathbb{E}_{Z|\bm{X}}\bigg[-\log |\bm{I}+\frac{1-\alpha}{\sigma_\epsilon^2}(\bm{K}_{\bm{f},\bm{f}}-\bm{Q})|^{\frac{-\alpha}{2(1-\alpha)}}  + \frac{\alpha\text{Tr}(\bm{K}_{\bm{f},\bm{f}}-\bm{Q})}{2\sigma_\epsilon^2}\bigg]\bigg]\\
        &\leq\mathbb{E}_{\bm{X}}\bigg[\frac{\alpha N}{2(1-\alpha)}\log\bigg[1+\frac{1-\alpha}{\sigma_\epsilon^2}\frac{[(M+1)\sum_{m=M+1}^N\lambda_m(\bm{K}_{\bm{f},\bm{f}})+2Nv\epsilon]}{N}\bigg]\bigg]+\\
        &\quad \alpha\frac{(M+1)\sum_{m=M+1}^N\lambda_m(\bm{K}_{\bm{f},\bm{f}})+2Nv\epsilon}{2\sigma_\epsilon^2}\bigg]\\
        &\leq\frac{\alpha N}{2(1-\alpha)}\log\bigg[1+\frac{1-\alpha}{\sigma_\epsilon^2}\frac{[(M+1)N\sum_{m=M+1}^\infty\lambda_m+2Nv\epsilon]}{N}\bigg]+\\
        &\quad \alpha\frac{(M+1)N\sum_{m=M+1}^\infty\lambda_m+2Nv\epsilon}{2\sigma_\epsilon^2}.
    \end{split}
\end{equation*}
By the Markov's inequality, we have the following bound with probability at least $1-\delta$ for any $\delta\in(0,1)$.
\begin{equation*}
    \begin{split}
        D_\alpha[p||q]&\leq\alpha\frac{(M+1)N\sum_{m=M+1}^\infty\lambda_m+2Nv\epsilon}{2\delta\sigma_\epsilon^2}+\\
        &\frac{1}{\delta}\frac{\alpha }{2(1-\alpha)}\log\bigg[1+\frac{1-\alpha}{\sigma_\epsilon^2}\frac{[(M+1)N\sum_{m=M+1}^\infty\lambda_m+2Nv\epsilon]}{N}\bigg]^N.
    \end{split}
\end{equation*}
\end{proof}
As $\alpha\to1$, we obtain the bound for the KL divergence.

\begin{theorem}
    Suppose $N$ data points are drawn i.i.d from input distribution $p(\bm{x})$ and $k(\bm{x},\bm{x})\leq v, \forall \bm{x}\in\mathcal{X}$. Sample $M$ inducing points from the training data with the probability assigned to any set of size $M$ equal to the probability assigned to the corresponding subset by an $\epsilon$ k-Determinantal Point Process (k-DPP) \citep{belabbas2009spectral} with $k=M$. With probability at least $1-\delta$,
    \begin{equation*}
        \begin{split}
            D_\alpha[q||p]&\leq\frac{1}{\delta}\frac{\alpha }{2(1-\alpha)}\log\bigg[1+\frac{1-\alpha}{\sigma_\epsilon^2}\frac{[(M+1)N\sum_{m=M+1}^\infty\lambda_m+2Nv\epsilon]}{N}\bigg]^N+\\
            &\quad \alpha\frac{(M+1)N\sum_{m=M+1}^\infty\lambda_m+2Nv\epsilon}{2\delta\sigma_\epsilon^2}\frac{\norm{\bm{y}}^2}{\sigma_\epsilon^2}
        \end{split}
    \end{equation*}
    where $C=N\sum_{m=M+1}^\infty\lambda_m$ and $\lambda_m$ are the eigenvalues of the integral operator $\mathcal{K}$ associated to kernel, $k$ and $p(\bm{x})$.
\end{theorem}
\begin{proof}

Using lemma in appendix, we have
\begin{equation*}
    \begin{split}
        D_\alpha[p||q]&\leq-\log |\bm{I}+\frac{1-\alpha}{\sigma_\epsilon^2}(\bm{K}_{\bm{f},\bm{f}}-\bm{Q})|^{\frac{-\alpha}{2(1-\alpha)}}  + \frac{1}{2}\norm{\bm{y}}^2\frac{\alpha\text{Tr}(\bm{K}_{\bm{f},\bm{f}}-\bm{Q})}{\sigma_{\epsilon}^4+\alpha\sigma_{\epsilon}^2\text{Tr}(\bm{K}_{\bm{f},\bm{f}}-\bm{Q})}\\
        &\leq-\log |\bm{I}+\frac{1-\alpha}{\sigma_\epsilon^2}(\bm{K}_{\bm{f},\bm{f}}-\bm{Q})|^{\frac{-\alpha}{2(1-\alpha)}}  + \frac{1}{2}\frac{\norm{\bm{y}}^2}{\sigma_{\epsilon}^2}\frac{\alpha\text{Tr}(\bm{K}_{\bm{f},\bm{f}}-\bm{Q})}{\sigma_{\epsilon}^2+\alpha\text{Tr}(\bm{K}_{\bm{f},\bm{f}}-\bm{Q})}\\
        &\leq-\log |\bm{I}+\frac{1-\alpha}{\sigma_\epsilon^2}(\bm{K}_{\bm{f},\bm{f}}-\bm{Q})|^{\frac{-\alpha}{2(1-\alpha)}}  + \frac{1}{2}\frac{\norm{\bm{y}}^2}{\sigma_{\epsilon}^2}\frac{\alpha\text{Tr}(\bm{K}_{\bm{f},\bm{f}}-\bm{Q})}{\sigma_{\epsilon}^2}.
    \end{split}
\end{equation*}

Following the same argument in the proof of Theorem \ref{bound1}, we have
\begin{equation*}
    \begin{split}
        &\frac{\alpha }{2(1-\alpha)}\log\bigg[1+\frac{1-\alpha}{\sigma_\epsilon^2}\frac{[(M+1)N\sum_{m=M+1}^\infty\lambda_m+2Nv\epsilon]}{N}\bigg]^N+\\
        &\quad \alpha\frac{(M+1)N\sum_{m=M+1}^\infty\lambda_m+2Nv\epsilon}{2\sigma_\epsilon^2}\frac{\norm{\bm{y}}^2}{\sigma_\epsilon^2}.
    \end{split}
\end{equation*}
\end{proof}
As $\alpha\to 1$, we reach the bound for the KL divergence.

\subsection{Risk Bound}
The Bayes risk is defined as $\mathcal{R}=\mathbb{E}[r(\bm{\theta},\bm{\theta}^*)]=\int r(\bm{\theta},\bm{\theta}^*) p_{\bm{\theta}}(\bm{f}|\bm{U},\bm{\mathcal{Z}}) d\bm{\theta}$ \citep{yang2017alpha}.
\begin{theorem}
    With probability at least $1-\delta$,
    \begin{align*}
        &\int_{\bm{\Theta}}\big\{r(\bm{\theta},\bm{\theta}^*) p_{\bm{\theta}}(\bm{f}|\bm{U},\bm{\mathcal{Z}}) \big\}d\bm{\theta}\\
        &\leq\frac{\alpha}{n(1-\alpha)}\bigg(\log \frac{p_{\bm{\theta}}(\bm{y})p_{\bm{\theta}^*}(\bm{y})}{\mathcal{N}(\bm{0}, \sigma_\epsilon^2 I+(1-\alpha)\bm{K}_{\bm{f},\bm{f}}+\alpha\bm{Q})} - \\
        &\log |\bm{I}+\frac{1-\alpha}{\sigma_\epsilon^2}(\bm{K}_{\bm{f},\bm{f}}-\bm{Q})|^{\frac{-\alpha}{2(1-\alpha)}}\bigg)+\frac{1}{n(1-\alpha)}\log\frac{1}{\delta}.
    \end{align*}
\end{theorem}
\begin{proof}

Let $p_{\bm{\theta}^*}\coloneqq p_{\bm{\theta}^*}(y)=p(\bm{y}|\bm{\theta}^*)$. Using Jensen's inequality, we have
\begin{align*}
    \mathbb{E}_{p_{\bm{\theta}^*}}&\bigg[\int q(\bm{f},\bm{U}|\bm{\mathcal{Z}})\exp\bigg\{\alpha\log\frac{p(\bm{f},\bm{U},\bm{y}|\bm{\mathcal{Z}})}{p(\bm{y}|\bm{\theta}^*)q(\bm{f},\bm{U}|\bm{\mathcal{Z}}))}\bigg\} d\bm{U}\bigg]\\
    &=\int p(\bm{y}|\bm{\theta}^*) \int q(\bm{f},\bm{U}|\bm{\mathcal{Z}})\bigg(\frac{p(\bm{f},\bm{U},\bm{y}|\bm{\mathcal{Z}})}{p(\bm{y}|\bm{\theta}^*)q(\bm{f},\bm{U}|\bm{\mathcal{Z}}))}\bigg)^\alpha d\bm{U} d\bm{y}\\
    &\leq \int p(\bm{y}|\bm{\theta}^*) \bigg(\int q(\bm{f},\bm{U}|\bm{\mathcal{Z}})\frac{p(\bm{f},\bm{U},\bm{y}|\bm{\mathcal{Z}})}{p(\bm{y}|\bm{\theta}^*)q(\bm{f},\bm{U}|\bm{\mathcal{Z}}))}\bigg)^\alpha d\bm{U} d\bm{y}\\
    &\leq \int p(\bm{y}|\bm{\theta}^*) \bigg( \frac{p(\bm{f},\bm{y}|\bm{\mathcal{Z}})}{p(\bm{y}|\bm{\theta}^*)}\bigg)^\alpha d\bm{y}\\
    &= e^{-(1-\alpha)D_\alpha[\bm{\theta}||\bm{\theta}^*]}.
\end{align*}
By re-arranging terms in inequality above and take expectation with respect to $p_{\bm{\theta}}$, we have
\begin{align*}
    \mathbb{E}_{p_{\bm{\theta}^*}}\bigg[\int p_{\bm{\theta}}(y) \int q(\bm{f},\bm{U}|\bm{\mathcal{Z}})\exp\bigg\{\alpha\log\frac{p(\bm{f},\bm{U},\bm{y}|\bm{\mathcal{Z}})}{p(\bm{y}|\bm{\theta}^*)q(\bm{f},\bm{U}|\bm{\mathcal{Z}}))} + (1-\alpha)D_\alpha[\bm{\theta}||\bm{\theta}^*] - \log\frac{1}{\delta}\bigg\} d\bm{U}d\bm{\theta} \bigg]\leq \delta.
\end{align*}
Using the variational dual representation of KL divergence, we have
\begin{align*}
    \mathbb{E}_{p_{\bm{\theta}^*}}&\exp\bigg[\int p_{\bm{\theta}}(y) \int q(\bm{f},\bm{U}|\bm{\mathcal{Z}})\bigg\{\alpha\log\frac{p(\bm{f},\bm{U},\bm{y}|\bm{\mathcal{Z}})}{p(\bm{y}|\bm{\theta}^*)q(\bm{f},\bm{U}|\bm{\mathcal{Z}}))} + (1-\alpha)D_\alpha[\bm{\theta}||\bm{\theta}^*] - \log\frac{1}{\delta}\bigg\} d\bm{U}d\bm{\theta} \bigg]\leq \delta.
\end{align*}
Using Markov inequality, we complete the proof.
\end{proof}

\section{Computation}
\label{sec:computation}
The likelihood function of R\'enyi $\mathcal{GP}$ can be efficiently optimized by the mBCG algorithms. 

\subsection*{On Computing Inverse}
$\bm{\Xi}^{-1}\bm{y}$ can be calculated by the conjugate gradient (CG) algorithm. Specifically, we solve the following quadratic optimization problem
\begin{align*}
    \bm{\Xi}^{-1}\bm{y}=\argmin_{\bm{u}} \bigg(\frac{1}{2}\bm{u}^T\bm{\Xi}\bm{u}-\bm{u}^T\bm{y}\bigg).
\end{align*}
Furthermore, CG can be extended to return a matrix output. Let $\bm{\Theta}=[\bm{y}\ \ \bm{K}_{\bm{f},\bm{U}}]$, then we can compute both $\bm{\Xi}^{-1}\bm{y}$ and $\bm{\Xi}^{-1}\bm{K}_{\bm{f},\bm{U}}$ by solving
\begin{align*}
    \bm{\Xi}^{-1}\bm{\Theta}=\argmin_{\bm{U}} \bigg(\frac{1}{2}\bm{U}^T\bm{\Theta}\bm{U}-\bm{U}^T\bm{\Theta}\bigg).
\end{align*}

\subsection*{On Computing Determinant}
$\log|\bm{\Xi}|$ can be computed in two ways. First, we can use pivoted Cholesky decomposition. Second, we can use Lanczos algorithm. When running Lanczos algorithm, we only need to return the Tridiagonal matrix $T$ and we have $\log|\bm{\Xi}|=\Tr(\log T)$.

\subsection*{On Computing Gradient}
Let $\bm{Z}=[\bm{z}_1,\ldots,\bm{z}_t]$ be a set of vectors where $\bm{z}_i$ is drawn from $\mathcal{N}(\bm{0},\bm{I})$. Then we can use mBCG to compute $\bm{\Xi}^{-1}\bm{Z}$ and calculate gradient as
\begin{align*}
    \Tr\big(\bm{\Xi}^{-1}\frac{d\bm{\Xi}}{d\bm{\theta}}\big)\approx\frac{1}{t}\sum_{i=1}^t(\bm{z}_i^T\bm{\Xi}^{-1})\bigg(\frac{d\bm{\Xi}}{d\bm{\theta}}\bm{z}_i\bigg).
\end{align*}
Please refer to \citet{gardner2018gpytorch} for the detailed implementation.

\bibliography{example_paper}
\bibliographystyle{icml2020}
\end{document}